\documentclass[twoside,11pt]{article}

\usepackage[preprint]{jmlr2e}
\usepackage{amsmath}
\usepackage{amssymb}
\usepackage{mathtools}
\usepackage{booktabs}
\usepackage{graphicx}
\usepackage{enumitem}
\usepackage{algorithm}
\usepackage{algpseudocode}
\usepackage{lastpage}
\usepackage{xcolor}
\usepackage{tikz-cd}

\makeatletter

\providecommand*{\toclevel@algorithm}{0}
\makeatother

\jmlrheading{0}{2026}{1-\pageref{LastPage}}{Draft}{Draft}{draft}{Ao Xu}
\ShortHeadings{Doeblin-Anchored Contrastive Learning}{Xu}
\firstpageno{1}

\newcommand{\X}{\mathsf X}
\newcommand{\BorelX}{\mathcal X}
\newcommand{\A}{\mathcal A}
\newcommand{\K}{\mathcal K}
\newcommand{\G}{\mathcal G}
\newcommand{\Prob}{\mathbb P}
\newcommand{\E}{\mathbb E}
\newcommand{\TV}{\mathrm{TV}}
\newcommand{\esssup}{\operatorname*{ess\,sup}}

\newcommand{\eps}{\varepsilon}

\newcommand{\norm}[1]{\left\lVert #1\right\rVert}
\newcommand{\abs}[1]{\left|#1\right|}
\newcommand{\bracks}[1]{\left[#1\right]}
\newcommand{\paren}[1]{\left(#1\right)}

\newcommand{\dd}{\,\mathrm d}
\newcommand{\KL}{\mathrm{KL}}
\newcommand{\Ber}{\mathrm{Ber}}

\newcommand{\Excess}{\mathcal E}
\newcommand{\Nint}{N_{\infty}^{\mathrm{int}}}
\newtheorem{assumption}[theorem]{Assumption}

\begin{document}

\title{A Doeblin-Anchored Contrastive Chart for Learning Markov Transition Kernels}

\author{\name Ao Xu\thanks{Corresponding author.} \email xuao24@mails.jlu.edu.cn \\
       \addr School of Artificial Intelligence\\
       Jilin University\\
       No. 2699, Qianjin Street, Chaoyang District\\
       Changchun 130012, China\\
       and \\
       \addr Zhongguancun Academy\\
       Daniufang 2nd Ring Road, Haidian District\\
       Beijing 100094, China}

\editor{Action Editor}

\maketitle

\begin{abstract}
Learning a Markov transition model is not merely conditional density estimation: the
learned object must be a valid transition kernel before it is iterated in downstream
dynamics.  This paper introduces a Doeblin-anchored contrastive chart, a
statistical-to-dynamical coordinate framework for learning transition kernels from
contrastive objectives.  Given a restart law and an anchor strength, the chart mixes the
target transition with the restart law.  The resulting anchored kernel is simultaneously
a Doeblin-minorized Markov kernel, the positive conditional law in a binary contrastive
experiment, and an explicitly invertible coordinate for the original transition law.  We
prove that the anchored contrastive risk identifies the anchored transition density
and calibrates excess risk to density error.  Since inversion of a learned score may
produce a signed or unnormalized object, we introduce a measurable Markovization
operator that restores kernel validity while preserving integrated $L^1$ accuracy up to
a constant factor.  Oracle inequalities and H\"older--ReLU approximation bounds yield
nonparametric rates for independent transition pairs.  For stationary geometrically
$\beta$-mixing trajectories, a conservative thinning-and-coupling extension yields the
same reconstruction interface with an effective sample size.  Occupancy-weighted
perturbation bounds transfer one-step kernel error to finite-horizon marginal,
path-law, and occupation-measure errors under explicit coverage.
\end{abstract}


\section{Introduction}
\label{sec:introduction}

Markov transition models are used not only for one-step prediction, but also for
simulation, multistep forecasting, occupation measures, and stationary summaries.  This
makes transition learning different from ordinary conditional density estimation.  A
statistical procedure may fit a conditional density under a convenient loss, while leaving
unclear whether the learned object is a valid transition kernel or which transition-kernel
metric is controlled.  Conversely, classical perturbation theory for Markov chains usually
starts from two valid kernels and studies how a given kernel error propagates to marginals,
path laws, or invariant distributions \citep{meyn2009markov,mitrophanov2005sensitivity}.
The problem considered here lies at the interface of these two viewpoints: how can a
transition law be learned through a contrastive statistical objective while still producing a
valid Markov kernel whose error has dynamical meaning?

Contrastive density estimation provides a natural statistical route.  Noise-contrastive
estimation reduces density learning to binary classification against a known reference law
\citep{gutmann2012nce}, and conditional variants extend this idea to conditional and
unnormalized models \citep{ceylan2018conditional}.  More broadly, density-ratio estimation
views such objectives as ways to learn ratios relative to a reference distribution
\citep{sugiyama2012density}.  Recent anchored contrastive constructions for ordinary
density estimation show that mixing a target density with a reference density can create an
interior object with a calibration inequality and an explicit de-anchoring map
\citep{li2026dannce}.  For Markov transition densities, however, the conditioning variable
and the downstream use of the learned object introduce additional structure.  It is not
enough to contrast each conditional density separately.  One must specify which conditional
law is learned, how it reconstructs the original transition kernel, whether the reconstruction
is itself a valid kernel, and how the resulting error propagates through the Markov dynamics.

This paper introduces a Doeblin-anchored contrastive chart for transition-kernel
learning.  Given a restart law $\nu$ and an anchor strength $\eps$, the chart maps a
transition kernel $K$ to the anchored kernel
\[
        A_{\eps,\nu}K=(1-\eps)K+\eps\nu .
\]
The point of this construction is not merely to regularize the transition density.  The
anchored object has three simultaneous interpretations.  First, it is a genuine Markov
kernel satisfying a Doeblin minorization, a classical condition in Markov chain theory
\citep{meyn2009markov,nummelin1978splitting}.  Second, it is the positive conditional
law in a binary contrastive experiment against the restart reference.  Third, it is an
explicitly invertible coordinate for the original transition law, with inverse
$(A-\eps\nu)/(1-\eps)$.  Thus the restart anchor provides a common coordinate system in
which contrastive estimation, kernel reconstruction, and Markov perturbation analysis act
on the same object.

This viewpoint turns contrastive transition learning into a validity-preserving
reconstruction problem.  A learned contrastive score estimates the anchored density, not
directly the original transition density.  Applying the inverse chart then produces a
de-anchored transition score, but this score need not be nonnegative or row-normalized.
This difficulty is largely absent from standard conditional density estimation, where the
target is usually a one-step density or score
\citep{hall1999methods,hyndman2002nonparametric,chen2007large}, but it is essential for
Markov dynamics, where the learned object must be iterated.  We therefore introduce a
deterministic Markovization map that clips and renormalizes each row.  The map restores
transition-kernel validity while preserving integrated $L^1$ accuracy up to a constant
factor.  The resulting estimator is therefore not only a contrastive score: it is a valid
Markov kernel to which finite-horizon perturbation bounds can be applied.

The main idea is summarized by the following diagram.
\begin{center}
\begin{tikzcd}[
  column sep=huge,
  row sep=large,
  cells={nodes={align=center}}
]
{$K_0$}
  \arrow[r,"{\;\mathrm{anchor}\;}"]
&
{$a_0=(1-\eps)k_0+\eps r$}
  \arrow[r,"{\;\mathrm{contrast}\;}"]
&
{contrastive risk $\mathcal R$}
  \arrow[d,"{\;\mathrm{ERM}\;}"]
\\
{\shortstack{finite-horizon\\[-0.2em]dynamical error}}
&
{$\widehat K_n$}
  \arrow[l,"{\;\mathrm{perturbation}\;}"']
&
{$\widehat a_n$}
  \arrow[l,"{\;\mathrm{reconstruct}\;}"']
\end{tikzcd}
\end{center}
The diagram should be read as a coordinate framework rather than as a single estimator.
The upper row defines the population chart, while the lower row describes the
reconstruction path followed by the estimator.  The analysis proves quantitative
stability along the arrows: contrastive excess risk controls anchored-density error,
de-anchoring and Markovization convert this into valid-kernel error, and
occupancy-weighted perturbation transfers valid-kernel error to finite-horizon
dynamical error.

Beyond the basic reconstruction theorem, the framework admits standard statistical and
dynamical extensions without changing the interface.  H\"older--ReLU approximation and
entropy bounds yield nonparametric transition-density rates, while a conservative
thinning-and-coupling argument extends the oracle-to-reconstruction interface to
stationary geometrically $\beta$-mixing trajectories through an effective sample size.  A
rare-state example clarifies why finite-horizon coverage assumptions cannot be dropped.

The contributions of the paper can be summarized as follows.

\begin{itemize}[leftmargin=2em]
\item \textbf{A contrastive Doeblin chart for transition kernels.}
We introduce an exact coordinate dictionary between original transition kernels,
Doeblin-minorized anchored kernels, contrastive posterior coordinates, and de-anchored
reconstruction coordinates.  The same restart law therefore serves simultaneously as a
Markov-chain minorization device, a contrastive reference law, and a reconstruction
coordinate.

\item \textbf{Risk geometry and statistical learning in the chart.}
We prove that the anchored contrastive risk has the anchored transition density as its
unique population minimizer, admits an integrated Bernoulli--KL representation, and
calibrates excess contrastive risk to anchored and de-anchored density error.  Standard
oracle inequalities and H\"older--ReLU approximation bounds can then be inserted into
this geometry to obtain nonparametric transition-density rates.  A conservative
thinning-and-coupling argument also extends the same oracle-to-reconstruction interface
to a thinned ERM based on stationary geometrically $\beta$-mixing trajectories,
replacing the nominal sample size by an effective sample size.

\item \textbf{Validity-preserving reconstruction.}
We show that applying the inverse chart to a learned score may leave the cone of Markov
kernels, producing a signed or unnormalized object.  To close this gap, we introduce a
measurable Markovization operator that restores nonnegativity and row normalization while
preserving integrated $L^1$ accuracy up to a constant factor.

\item \textbf{Dynamical transfer and diagnostic validation.}
We connect the reconstructed kernel to downstream Markov dynamics by combining the
one-step reconstruction error with occupancy-weighted perturbation bounds.  This yields
finite-horizon marginal, path-law, and occupation-measure guarantees under explicit
coverage.  The experiments are designed as diagnostics for the full
statistical-to-dynamical interface, including calibration, validity repair,
anchor-strength tradeoffs, trajectory stress tests, and coverage-failure examples.
\end{itemize}

\section{Related Work}
\label{sec:related-work}

This section positions the paper relative to six lines of work: contrastive density
estimation, conditional and transition density estimation, Doeblin minorization and
Markov perturbation theory, learning stochastic dynamics, oracle inequalities, and
neural sieve rates.

\paragraph{Noise-contrastive and conditional contrastive density estimation.}
Noise-contrastive estimation reduces density estimation to binary classification against
a known reference law \citep{gutmann2012nce}, with conditional variants allowing the
noise distribution to depend on the covariate \citep{ceylan2018conditional}.  Density-ratio
estimation gives a broader perspective on such objectives \citep{sugiyama2012density}.
Recent data-augmented contrastive methods anchor an ordinary density to a reference
density before applying a contrastive risk, yielding calibration and an explicit
de-anchoring map \citep{li2026dannce}.  The present paper uses the reference law in a
transition setting where it is simultaneously the contrastive reference and the restart
law in the Doeblin anchor $A_{\eps,\nu}K=(1-\eps)K+\eps\nu$.

\paragraph{Transition density and conditional density estimation.}
Nonparametric conditional density estimation includes kernel and local methods
\citep{hall1999methods,hyndman2002nonparametric,hall2004nonparametric}, single-index
and other structural approaches \citep{hall2005conditional}, and general sieve
estimators \citep{chen1998sieve,chen2007large}.  Transition-density estimation for
Markov processes has also been studied in econometrics and time series, including
diffusion likelihood expansions \citep{aitsahalia2002maximum} and spectral methods for
stationary Markov processes \citep{hansen1998spectral}.  Minimax theory for conditional
density estimation identifies the role of smoothness in both variables
\citep{li2022minimax}, while neural conditional density models include mixture density
networks, conditional flows, and basis-expansion estimators
\citep{bishop1994mixture,trippe2018conditional,izbicki2017flexcode,gao2022lincde}.
These works primarily target one-step density accuracy; the additional issue here is
that the learned object must be reconstructed as a valid transition kernel before it is
iterated.

\paragraph{Doeblin minorization and Markov perturbation theory.}
Doeblin minorization is a classical device in Markov chain theory
\citep{meyn2009markov}.  Restart and regeneration constructions, including Nummelin
splitting, use related affine mixtures to analyze Harris recurrent chains
\citep{nummelin1978splitting}.  The Dobrushin coefficient describes contraction and
forgetting of initial conditions \citep{dobrushin1956central}, while sensitivity of
invariant laws under kernel perturbations has been developed for uniformly ergodic and
operator-theoretic regimes \citep{mitrophanov2005sensitivity,kartashov1996strong}.  In
this paper the same affine restart map is used as a statistical coordinate: the anchored
kernel is the positive class in the contrastive experiment, and de-anchoring is the
explicit inverse coordinate map.

\paragraph{Learning stochastic dynamics.}
Model-based reinforcement learning and probabilistic forecasting learn stochastic
transition models for rollout, planning, or simulation \citep{ha2018world,hafner2020dream,
chua2018deep}.  Rollout degradation under learned models is well known empirically
\citep{talvitie2014model}, and one-step-to-multistep error bounds appear in model-based
policy optimization \citep{janner2019when}.  The present work is not a new dynamics
algorithm or benchmark; it studies how a contrastive statistical objective can produce a
valid kernel to which finite-horizon perturbation bounds apply under explicit coverage.

\paragraph{Oracle inequalities and fast rates.}
Empirical risk minimization with localized complexity measures, Rademacher bounds,
Bernstein conditions, and validation-based selection is a standard route to oracle
inequalities and fast rates \citep{bartlett2002rademacher,bartlett2005local,
tsybakov2004optimal,koltchinskii2006local,lecue2012rates}.  For dependent observations,
blocking and coupling under absolute regularity provide a classical route from
trajectory samples to effective independent sample sizes; see \citet{doukhan1994mixing}.  These tools supply the statistical front end of the
paper once the anchored contrastive risk has been placed in a bounded calibrated chart.

\paragraph{Neural sieve rates and minimax lower bounds.}
Deep ReLU networks achieve near-minimax rates for H\"older function classes under
standard approximation and compositional assumptions \citep{yarotsky2017relu,
schmidthieber2020deep}.  Covering-number bounds for neural network classes connect these
approximations to estimation error \citep{anthony1999neural,bartlett2019nearly}, while
classical minimax theory supplies matching lower-bound constructions
\citep{stone1982rates}.  The H\"older--ReLU part of the paper instantiates these ideas
for bounded transition densities and then transports the resulting rate through
de-anchoring and Markovization.

\paragraph{Summary of positioning.}
The individual components above are classical or already well developed in isolation:
contrastive estimation, Doeblin minorization, empirical-process oracle inequalities,
Markov perturbation bounds, and neural sieve rates.  The contribution here is the
interface that makes them act on the same transition-learning object.  The restart anchor
creates a chart between original kernels, Doeblin-minorized kernels, and contrastive
posterior coordinates; risk calibration controls the anchored density; de-anchoring and
Markovization produce a valid kernel for the original transition law; and
occupancy-weighted perturbation transfers one-step kernel error to finite-horizon
dynamical error under explicit coverage.  The trajectory extension and rare-state
example clarify the sampling and coverage boundaries of this interface.

\section{Setup}
\label{sec:setup}

\subsection{Dominated transition kernels}

Let $(\X,\BorelX)$ be a standard Borel measurable space.  Throughout the statistical
results we additionally assume that $\lambda$ is a probability measure on
$(\X,\BorelX)$ and that all transition kernels considered below are dominated by
$\lambda$.  A Markov kernel $K$ from $\X$ to $\X$ is a map
$K:\X\times\BorelX\to[0,1]$ such that $K(x,\cdot)$ is a probability measure for every
$x$ and $x\mapsto K(x,B)$ is measurable for every $B\in\BorelX$.  We write
\[
        K(x,\dd y)=k(y\mid x)\lambda(\dd y)
\]
when $k:\X\times\X\to[0,\infty)$ is jointly measurable and
$\int k(y\mid x)\lambda(\dd y)=1$ for every $x$.  Once a present-state design law $\mu$
is fixed below, equalities between densities are understood
$\mu\otimes\lambda$-almost everywhere unless a pointwise representative is explicitly
specified.

The present-state design law is a probability measure $\mu$ on $(\X,\BorelX)$.  The
target transition kernel is denoted by $K_0$, with density $k_0$.  The independent-pair
observation model used for the first oracle bounds in
Sections~\ref{sec:estimation}--\ref{sec:relu} is
\[
        X_i\sim\mu,\qquad Y_i\mid X_i\sim K_0(X_i,\cdot),
        \qquad i=1,\ldots,n,
\]
with the pairs $(X_i,Y_i)$ independent.  Section~\ref{sec:dynamics} gives the
dynamical transfer consequences.

\subsection{Reference restart law and metrics}

Let $\nu$ be a probability measure dominated by $\lambda$ with density
$r:\X\to[0,\infty)$:
\[
        \nu(\dd y)=r(y)\lambda(\dd y).
\]
The law $\nu$ does not depend on the current state.  This state-independent choice is
what makes the anchor below a restart kernel and yields a Doeblin minorization.  For a
measurable function $f$ on $\X\times\X$, define
\[
 \norm{f}_{2,\mu\lambda}^2
   :=\int f(x,y)^2\,\mu(\dd x)\lambda(\dd y),\qquad
 \norm{f}_{1,\mu\lambda}
   :=\int \abs{f(x,y)}\,\mu(\dd x)\lambda(\dd y).
\]
For a bounded measurable $f$ on $\X\times\X$, write
\[
        \norm{f}_{\infty}:=\sup_{(x,y)\in\X\times\X}\abs{f(x,y)}.
\]
For kernels $K$ and $L$, define the integrated and uniform total-variation semimetrics
\begin{align*}
 d_{\mu,\TV}(K,L)
   &:=\int \TV\paren{K(x,\cdot),L(x,\cdot)}\,\mu(\dd x),\\
 d_{\infty,\TV}(K,L)
   &:=\sup_{x\in\X}\TV\paren{K(x,\cdot),L(x,\cdot)}.
\end{align*}
If $K$ and $L$ have densities $k$ and $\ell$, then
\[
 d_{\mu,\TV}(K,L)=\frac12\norm{k-\ell}_{1,\mu\lambda}.
\]
Here $\TV(P,Q):=\sup_{B\in\BorelX}\abs{P(B)-Q(B)}$ for probability measures on
$(\X,\BorelX)$.  Because $\X$ is standard Borel, $\BorelX$ is countably generated and
the supremum defining total variation may be taken over a countable determining algebra.
The integrated metrics describe the design distribution used for learning.  The uniform
metric is stronger and will be required for general path-law perturbations in
Section~\ref{sec:dynamics}.
Measurability of the map $x\mapsto \TV(K(x,\cdot),L(x,\cdot))$ is standard on
standard Borel spaces; for completeness see Lemma~\ref{lem:tv-measurability} in
Appendix~\ref{app:tv}.

\subsection{Choice of restart law and anchor strength}
\label{sec:anchor-choice}

The restart law $\nu$ and anchor strength $\eps$ are design parameters of the chart.
The theory should not be read as saying that all choices are equivalent.  The restart
law plays three roles at once: it is the negative class in the contrastive experiment,
the lower envelope in the anchored density, and the fallback law used by the
Markovization map on degenerate rows.  A useful restart law should therefore be
simulable, should have support on the region where transition reconstruction is
desired, and should not put negligible mass on states that are important for the
downstream occupancies considered in Section~\ref{sec:dynamics}.

The anchor strength has the complementary tradeoff.  Increasing $\eps$ adds more
reference mass to the anchored density and improves the lower envelope
$a_0(y\mid x)\ge \eps r(y)$, which is favorable for contrastive curvature.  At the same
time, the inverse map divides by $1-\eps$, so reconstruction bounds carry factors such
as $(1-\eps)^{-1}$ or $(1-\eps)^{-2}$.  Thus $\eps$ should be treated as a genuine
statistical design parameter, not as a harmless technical constant.  In applications,
a conservative default is to use a full-support diffuse reference, or a mixture of an
empirical next-state marginal with such a diffuse reference, and to select $\eps$ using
validation contrastive risk together with the invalidity diagnostics before
Markovization.

\section{Contrastive Doeblin Charts}
\label{sec:charts}

\begin{definition}[Restart anchor]
\label{def:anchor}
For $\eps\in(0,1)$ and a restart law $\nu$, the restart anchor of a Markov kernel $K$ is
the Markov kernel
\[
   A_{\eps,\nu}K(x,B)
      :=(1-\eps)K(x,B)+\eps\nu(B),
      \qquad x\in\X,\ B\in\BorelX.
\]
If $K$ has density $k$ and $\nu$ has density $r$, then $A_{\eps,\nu}K$ has density
\[
        a_K(y\mid x):=(1-\eps)k(y\mid x)+\eps r(y).
\]
\end{definition}

\begin{theorem}[Contrastive Doeblin chart]
\label{thm:chart}
Let $\eps\in(0,1)$ and let $\nu$ be a probability measure on $(\X,\BorelX)$.
\begin{enumerate}[label=(\roman*),leftmargin=2em]
  \item The image of $A_{\eps,\nu}$ is exactly the class of Markov kernels $A$
  satisfying
  \[
        A(x,B)\ge \eps\nu(B),\qquad x\in\X,\ B\in\BorelX.
  \]
  For such an $A$, the inverse kernel is
  \[
        A_{\eps,\nu}^{-1}A(x,B)
        =\frac{A(x,B)-\eps\nu(B)}{1-\eps}.
  \]
  \item For all Markov kernels $K$ and $L$,
  \[
     d_{\mu,\TV}(A_{\eps,\nu}K,A_{\eps,\nu}L)
       =(1-\eps)d_{\mu,\TV}(K,L)
  \]
  for every design law $\mu$, and
  \[
     d_{\infty,\TV}(A_{\eps,\nu}K,A_{\eps,\nu}L)
       =(1-\eps)d_{\infty,\TV}(K,L).
  \]
  \item Suppose now that $\nu(\dd y)=r(y)\lambda(\dd y)$ with $r(y)>0$ for
  $\lambda$-almost every $y$, and that $A$ has a chosen jointly measurable density
  representative $a(y\mid x)$.  Fix a contrastive ratio $\tau>0$ and define the
  posterior coordinate
  \[
        \eta_{A,a}(x,y)
        :=
        \begin{cases}
          \dfrac{a(y\mid x)}{a(y\mid x)+\tau r(y)}, & r(y)>0,\\
          0, & r(y)=0
        \end{cases}
  \]
  on $\X\times\X$.  Then $a$ is recovered from $\eta_{A,a}$ on the set where $r(y)>0$ by
  \[
        a(y\mid x)=\tau r(y)\frac{\eta_{A,a}(x,y)}{1-\eta_{A,a}(x,y)}.
  \]
  If $A=A_{\eps,\nu}K$, then the density of $K$ is recovered on the same set by
  \[
        k(y\mid x)
          =\frac{\tau r(y)\eta_{A,a}(x,y)/(1-\eta_{A,a}(x,y))-\eps r(y)}
                 {1-\eps}
  \]
  Consequently, these density equalities hold $\mu\otimes\lambda$-almost everywhere
  for every design law $\mu$.
\end{enumerate}
\end{theorem}
\noindent The proof of Theorem~\ref{thm:chart} is given in Appendix~\ref{app:chart}.

\begin{remark}[Anchor tradeoff and selection]
\label{rem:epsilon-tradeoff}
The anchor is stabilizing in the forward coordinate and amplifying in the inverse
coordinate.  If $\widehat a$ is an anchored density score and
$\widetilde k_{\widehat a}=(\widehat a-\eps r)/(1-\eps)$, then
\[
        \norm{\widetilde k_{\widehat a}-k_0}_{p,\mu\lambda}
        =\frac{1}{1-\eps}\norm{\widehat a-a_0}_{p,\mu\lambda},
        \qquad p\in\{1,2\}.
\]
Larger $\eps$ supplies more restart mass to the positive contrastive density, but the
inverse map becomes unstable as $\eps\uparrow1$.  Reference support improves the
next-state contrastive coordinate; it does not by itself imply present-state occupancy
coverage.  In applications, $\nu$ and $\eps$ should therefore be treated as design
parameters, with $\eps$ selected by validation contrastive risk and pre-Markovization
invalidity diagnostics.
\end{remark}

Kernel identities in parts (i) and (ii) are pointwise in $x$ and in measurable sets $B$.
Density and posterior identities in part (iii) concern chosen jointly measurable density
representatives and the derived coordinate $\eta_{A,a}$; they are asserted almost
everywhere.

\section{Risk Geometry, Oracle Learning, and Kernel Reconstruction}
\label{sec:estimation}

This section advances the first two links in the main chain:
\[
 \text{contrastive risk}
 \;\Rightarrow\;
 \text{anchored density}
 \;\Rightarrow\;
 \text{valid Markov kernel}.
\]
We first define the augmented contrastive risk and estimator, then state the risk
geometry theorem, the oracle inequalities, and the Markovization result.

\subsection{Augmented transition pairs and empirical risk}

Write
\[
        a_0(y\mid x):=(1-\eps)k_0(y\mid x)+\eps r(y)
\]
for the anchored target density.  To estimate $a_0$ from samples of $K_0$, we simulate
restart variables
\[
        \widetilde Y_i\sim \nu,\qquad
        Z_{ij}\sim\nu,\quad j=1,\ldots,\tau,
\]
independently of the observed pairs and of each other.  In the empirical construction
below $\tau$ is a positive integer.  The variable $\widetilde Y_i$ supplies the anchored
positive reference mass, while the variables $Z_{ij}$ supply negative reference draws.
The current state is kept at $X_i$ in all terms.  This implementation has expectation
equal to a binary experiment whose positive conditional density is $a_0(\cdot\mid X_i)$
and whose negative conditional density is $r$.

In the statistical results the contrastive ratio $\tau$ is a fixed positive integer.
For any measurable candidate $a:\X\times\X\to(0,\infty)$ define
\[
  \eta_a(x,y):=\frac{a(y\mid x)}{a(y\mid x)+\tau r(y)},
\]
and the log-loss terms
\[
  \ell_a^+(x,y):=-\log \eta_a(x,y),\qquad
  \ell_a^-(x,y):=-\log\paren{1-\eta_a(x,y)}.
\]
The population risk is
\begin{align}
  \mathcal R(a)
   := {} &(1-\eps)\E\bracks{\ell_a^+(X,Y)}
          +\eps\E\bracks{\ell_a^+(X,\widetilde Y)}
          +\tau\E\bracks{\ell_a^-(X,Z)},
  \label{eq:population-risk}
\end{align}
where $X\sim\mu$, $Y\mid X\sim K_0(X,\cdot)$, and
$\widetilde Y,Z\sim\nu$ are independent of $X$.  Equivalently,
\begin{align}
  \mathcal R(a)
   =\int_{\X\times\X}
      \Bigg[
       -a_0(y\mid x)\log\frac{a(y\mid x)}{a(y\mid x)+\tau r(y)}
       -\tau r(y)\log\frac{\tau r(y)}{a(y\mid x)+\tau r(y)}
      \Bigg]
      \mu(\dd x)\lambda(\dd y).
  \label{eq:population-risk-integral}
\end{align}
The losses and risks are understood as extended-real quantities with
$-\log 0=+\infty$.  Under Assumptions~\ref{ass:bounds} and
\ref{ass:candidate-bounds} below, all risk values used in the theorems are finite.
Given observations and simulated references, the empirical risk is
\begin{align}
 \widehat{\mathcal R}_n(a)
   :=\frac1n\sum_{i=1}^n
      \Bigg[
        (1-\eps)\ell_a^+(X_i,Y_i)
        +\eps\ell_a^+(X_i,\widetilde Y_i)
        +\sum_{j=1}^{\tau}\ell_a^-(X_i,Z_{ij})
      \Bigg].
 \label{eq:empirical-risk}
\end{align}
By construction, $\E\widehat{\mathcal R}_n(a)=\mathcal R(a)$ whenever the expectation is
finite.

\subsection{Assumptions}
\begin{assumption}[Reference and target bounds]
\label{ass:bounds}
There are constants $0<r_- \le r_+<\infty$ and $M_0<\infty$ such that
\[
        r_-\le r(y)\le r_+
        \quad\text{for every }y\in\X,
        \qquad
        0\le k_0(y\mid x)\le M_0
\]
for $\mu\otimes\lambda$-almost every $(x,y)$.
\end{assumption}

\begin{assumption}[Candidate bounds]
\label{ass:candidate-bounds}
There are constants $0<\gamma\le\Gamma<\infty$ such that every
$a\in\A_n$ is jointly measurable and satisfies
\[
        \gamma\le a(y\mid x)\le \Gamma
\]
for every $(x,y)\in\X\times\X$.
\end{assumption}

Assumption~\ref{ass:bounds} implies
\[
   \eps r_-\le a_0(y\mid x)\le (1-\eps)M_0+\eps r_+
\]
almost everywhere.  The lower bound is supplied by the restart anchor even if the
original transition density vanishes.

\subsection{Estimator and reconstruction procedure}

Let $\A_n$ be a class of positive jointly measurable candidate functions.  For an
optimization tolerance $\delta_{\mathrm{opt}}\ge0$, an anchored empirical estimator is
any measurable $\widehat a_n\in\A_n$ satisfying
\[
        \widehat{\mathcal R}_n(\widehat a_n)
        \le
        \inf_{a\in\A_n}\widehat{\mathcal R}_n(a)
        +\delta_{\mathrm{opt}}.
\]
The direct de-anchored density score is
\[
        \widetilde k_n(y\mid x)
        :=\frac{\widehat a_n(y\mid x)-\eps r(y)}{1-\eps}.
\]
If $\widehat a_n$ lies in the image of the anchor map, then
$\widetilde k_n$ is already a transition density.  For general function classes it may
be signed or may fail to integrate to one.  Theorem~\ref{thm:markovization} below gives
a measurable Markovization operator that restores a valid kernel with a controlled
$L^1$ loss.

Algorithm~\ref{alg:anchored-contrastive} summarizes this estimator-to-kernel
conversion.  The display is meant to fix the order of the three operations used in the
theory: simulate the restart draws that define the empirical contrastive risk, fit an
anchored score by ERM, and only then apply the deterministic de-anchoring and
Markovization maps.  Thus the output of the statistical step is a score
$\widehat a_n$, while the output used for dynamics is the valid kernel
$\widehat K_n$.

\begin{algorithm}[t]
\caption{Anchored contrastive estimation and reconstruction}
\label{alg:anchored-contrastive}
\begin{algorithmic}[1]
\Require Transition pairs $(X_i,Y_i)_{i=1}^n$; simulable restart law $\nu$ with
density $r$; anchor strength $\eps$; number of negatives $\tau$; candidate class
$\A_n$; optimization tolerance $\delta_{\mathrm{opt}}$.
\Ensure Markov transition kernel estimator $\widehat K_n$.
\State \textbf{// Step 1: Sample anchor positives and reference negatives}
\For{$i=1$ \textbf{to} $n$}
  \State Draw $\widetilde Y_i\sim\nu$ independently of $(X_i,Y_i)$.
  \State Draw $Z_{i1},\ldots,Z_{i\tau}\stackrel{\mathrm{i.i.d.}}{\sim}\nu$
  independently of all observed and anchor-positive variables.
\EndFor
\State \textbf{// Step 2: Fit the anchored contrastive score}
\State Choose a measurable $\widehat a_n\in\A_n$ satisfying
\[
        \widehat{\mathcal R}_n(\widehat a_n)
        \le
        \inf_{a\in\A_n}\widehat{\mathcal R}_n(a)
        +\delta_{\mathrm{opt}}.
\]
\State \textbf{// Step 3: De-anchor and Markovize}
\State Set
\[
        \widetilde k_n(y\mid x)
        \leftarrow
        \frac{\widehat a_n(y\mid x)-\eps r(y)}{1-\eps}.
\]
\State Set $\widehat k_n\leftarrow\mathfrak M\widetilde k_n$.
\State Set $\widehat K_n(x,\dd y)\leftarrow \widehat k_n(y\mid x)\lambda(\dd y)$.
\State \Return $\widehat K_n$.
\end{algorithmic}
\end{algorithm}

\subsection{Risk Geometry of the Anchored Contrastive Chart}

For positive measurable $a$ for which the risks are finite, set
\[
        \Excess(a):=\mathcal R(a)-\mathcal R(a_0).
\]
Also define the integrated Hellinger-type discrepancy
\[
        H_{\mu\lambda}^2(a,a_0)
        :=\int_{\X\times\X}
          \paren{\sqrt{a(y\mid x)}-\sqrt{a_0(y\mid x)}}^2
          \mu(\dd x)\lambda(\dd y).
\]

\begin{theorem}[Risk Geometry of the Anchored Contrastive Chart]
\label{thm:calibration}
Suppose Assumptions~\ref{ass:bounds} and \ref{ass:candidate-bounds} hold.  Then:
\begin{enumerate}[label=(\roman*),leftmargin=2em]
  \item \textbf{Identification.}  The anchored target density $a_0$ uniquely minimizes
  $\mathcal R$ over all positive measurable $a$ for which the risk in
  \eqref{eq:population-risk-integral} is finite, up to $\mu\otimes\lambda$-almost-everywhere
  equality.
  \item \textbf{$L^2$ calibration.}  There exist constants $0<c_-\le c_+<\infty$,
  depending only on $(\eps,\tau,r_-,r_+,M_0,\gamma,\Gamma)$, such that for every
  $a\in\A_n$,
  \[
    c_-\norm{a-a_0}_{2,\mu\lambda}^2
      \le \Excess(a)
      \le c_+\norm{a-a_0}_{2,\mu\lambda}^2.
  \]
  Consequently,
  \[
     \norm{\widetilde k_a-k_0}_{2,\mu\lambda}^2
        \le \frac{1}{(1-\eps)^2c_-}\Excess(a),
  \]
  where $\widetilde k_a:=(a-\eps r)/(1-\eps)$.
  \item \textbf{Contrastive KL representation.}  Writing
  $\eta_0(x,y):=a_0(y\mid x)/\{a_0(y\mid x)+\tau r(y)\}$, the excess risk has the exact
  representation
  \begin{align}
   \Excess(a)
     =\int_{\X\times\X}
         \paren{a_0(y\mid x)+\tau r(y)}
         \KL\!\left(
            \Ber(\eta_0(x,y))
            \,\middle\|\,
            \Ber(\eta_a(x,y))
         \right)
         \mu(\dd x)\lambda(\dd y).
   \label{eq:contrastive-kl}
  \end{align}
  \item \textbf{$L^1$ and Hellinger consequences.}
  \begin{align}
   \norm{a-a_0}_{1,\mu\lambda}^2
        &\le c_-^{-1}\Excess(a), \label{eq:l1-calibration}\\
   H_{\mu\lambda}^2(a,a_0)
        &\le \frac{1}{4m_a c_-}\Excess(a),
        \qquad
        m_a:=\min\{\eps r_-,\gamma\}. \label{eq:hellinger-calibration}
  \end{align}
  If $a$ is itself a transition density and $A_a(x,\dd y)=a(y\mid x)\lambda(\dd y)$,
  then
  \[
        d_{\mu,\TV}(A_a,A_{\eps,\nu}K_0)^2
        \le \frac{1}{4c_-}\Excess(a).
  \]
\end{enumerate}
\end{theorem}
The proof of Theorem~\ref{thm:calibration} is given in Appendix~\ref{app:risk}.
The uniform boundedness assumptions keep the curvature and concentration statements
global; tail-localized extensions would require standard truncation arguments
\citep{wainwright2019high}.

\subsection{Oracle Inequalities for Anchored Contrastive ERM}
To formally analyze the statistical properties of our estimator, it is mathematically convenient to package the interdependent sample components from Section 3 into a single independent observation block. Specifically, for a block
\[
B=(X,Y,\widetilde Y,Z_1,\ldots,Z_\tau)
\]
distributed as described previously, define
\[
 g_a(B)
   :=(1-\eps)\ell_a^+(X,Y)
      +\eps\ell_a^+(X,\widetilde Y)
      +\sum_{j=1}^{\tau}\ell_a^-(X,Z_j).
\]
Let $\G_n:=\{g_a:a\in\A_n\}$. In this block-based notation, the empirical and population contrastive risks are precisely the sample average and expectation of this block-wise loss, respectively:
\[
  \widehat{\mathcal R}_n(a) := \frac{1}{n}\sum_{i=1}^n g_a(B_i),
  \qquad
 \mathcal R(a) := \E[g_a(B)].
\]
Throughout this subsection, the blocks $B_1,\ldots,B_n$ are independent copies of $B$
unless a different sampling scheme is stated explicitly.
Set the expected uniform deviation as
\[
  \Delta_n
 :=\E\sup_{a\in\A_n}
 \abs{\widehat{\mathcal R}_n(a)-\mathcal R(a)}.
\]
We assume that the displayed supremum is measurable.  Equivalently, each expectation in
Theorem~\ref{thm:oracle} may be read as an outer expectation.

For a class $\A$ of functions on $\X\times\X$ and $\delta>0$, write
$\Nint(\delta,\A)$ for the internal sup-norm covering number: it is the smallest
$N$ for which there exist $a_1,\ldots,a_N\in\A$ such that
\[
  \sup_{a\in\A}\min_{j=1,\ldots,N}\norm{a-a_j}_\infty\le\delta.
\]

\begin{theorem}[Oracle Inequalities for Anchored Contrastive ERM]
\label{thm:oracle}
Under Assumptions~\ref{ass:bounds} and \ref{ass:candidate-bounds}, let
$\widehat a_n$ be a measurable $\delta_{\mathrm{opt}}$-approximate empirical risk
minimizer over $\A_n$.  Then:
\begin{enumerate}[label=(\roman*),leftmargin=2em]
  \item \textbf{Global oracle bound.}  
    \[     
    \E\norm{\widetilde k_{\widehat a_n}-k_0}_{2,\mu\lambda}^2    
     \le       
       \frac{2\Delta_n+\delta_{\mathrm{opt}}}{(1-\eps)^2c_-}        
       +\frac{c_+}{(1-\eps)^2c_-}           
          \inf_{a\in\A_n}\norm{a-a_0}_{2,\mu\lambda}^2,  
    \]  
    where $c_-$ and $c_+$ are the calibration constants in Theorem~\ref{thm:calibration}. 
    Here, the expected uniform deviation $\Delta_n$ can be explicitly upper-bounded 
    for any $\delta > 0$ by
    \[
      \Delta_n
      \le
      2L_{\ell}\delta
      +
      B_{\ell}
      \sqrt{\frac{2\log\paren{2N^{\mathrm{int}}_\infty(\delta,\mathcal A_n)}}{n}},
    \]
    where $N^{\mathrm{int}}_\infty(\delta,\mathcal A_n)=\Nint(\delta,\mathcal A_n)$
    is the internal sup-norm covering number defined above.
  \item \textbf{Bernstein geometry.}  There exist finite constants $V_{\ell}$ and $B_h$,
  depending only on $(\eps,\tau,r_-,r_+,M_0,\gamma,\Gamma)$, such that for every $a\in\A_n$,
  \[
          \E\bracks{\paren{g_a(B)-g_{a_0}(B)}^2}
          \le V_{\ell}\Excess(a),
          \qquad
          \abs{g_a(B)-g_{a_0}(B)}\le B_h \ \text{a.s.}
  \]
  Here the blocks $(B_i)_{i=1}^n$ are i.i.d.
  \item \textbf{Localized fast-rate oracle bound.}  Let $\delta>0$ and
  $N:=\Nint(\delta,\A_n)<\infty$.  Put $u:=\log(2N)+\log(2/\eta)$ for $\eta\in(0,1)$.
  There is a constant $C_{\mathrm{fast}}<\infty$, depending only on
  $(\eps,\tau,r_-,r_+,M_0,\gamma,\Gamma)$, such that with probability at least $1-\eta$,
  \begin{equation}
 \Excess(\widehat a_n)
 \le
 \frac53\inf_{a\in\A_n}\Excess(a)
 +
 C_{\mathrm{fast}}
 \left[
    L_{\ell}\delta
    +\delta_{\mathrm{opt}}
    +\frac{\log\{2N^{\mathrm{int}}_\infty(\delta,\A_n)/\eta\}}{n}
 \right],
\label{eq:fast-oracle-risk}
\end{equation}
  Consequently, on the same event,
\[
   \norm{\widetilde k_{\widehat a_n}-k_0}_{2,\mu\lambda}^2
   \le
   \frac{C_{\mathrm{fast}}}{(1-\eps)^2c_-}
   \left[
      \inf_{a\in\A_n}\Excess(a)
      +L_{\ell}\delta
      +\delta_{\mathrm{opt}}
      +
      \frac{
        \log\{2N^{\mathrm{int}}_\infty(\delta,\A_n)/\eta\}
      }{n}
   \right].
\]
\end{enumerate}
\end{theorem}
The proof of Theorem~\ref{thm:oracle} is given in Appendix~\ref{app:erm}.

\subsection{Valid-Kernel Reconstruction by Markovization}
The next theorem is deterministic.  It clarifies how a de-anchored score can be
converted to a Markov kernel without asserting that every positive network score lies in
the anchor image.

\begin{definition}[Markovization]
\label{def:markovization}
Let $g:\X\times\X\to\mathbb R$ be jointly measurable and assume
$\int \abs{g(y\mid x)}\lambda(\dd y)<\infty$ for every $x$.  Write
$g^+(y\mid x)=\max\{g(y\mid x),0\}$ and
$c_g(x)=\int g^+(y\mid x)\lambda(\dd y)$.  Define
\[
   \mathfrak M g(y\mid x)
   :=
   \begin{cases}
     g^+(y\mid x)/c_g(x), & c_g(x)>0,\\
     r(y), & c_g(x)=0.
   \end{cases}
\]
\end{definition}

\begin{theorem}[Valid-Kernel Reconstruction by Markovization]
\label{thm:markovization}
The function $\mathfrak M g$ in Definition~\ref{def:markovization} is a jointly
measurable transition density.  For every transition density $k$,
\[
        \int\abs{\mathfrak M g(y\mid x)-k(y\mid x)}\lambda(\dd y)
        \le
        2\int\abs{g(y\mid x)-k(y\mid x)}\lambda(\dd y)
        \qquad\text{for every }x.
\]
Consequently,
\[
        \norm{\mathfrak M g-k}_{1,\mu\lambda}
        \le 2\norm{g-k}_{1,\mu\lambda}.
\]
In particular, on the event of Theorem~\ref{thm:oracle}(iii), the Markovized
de-anchored estimator
\[
        \widehat k_n:=\mathfrak M\widetilde k_{\widehat a_n},
        \qquad
        \widehat K_n(x,\dd y):=\widehat k_n(y\mid x)\lambda(\dd y)
\]
satisfies
\begin{align}
 d_{\mu,\TV}(\widehat K_n,K_0)^2
 \le
 \frac{C_{\mathrm{fast}}}{(1-\eps)^2c_-}
        \left[
          \inf_{a\in\A_n}\Excess(a)
          +L_{\ell}\delta
          +\delta_{\mathrm{opt}}
          +\frac{\log(2\Nint(\delta,\A_n))+\log(2/\eta)}{n}
        \right].
 \label{eq:markovized-fast-rate}
\end{align}
The bound is in integrated total variation for the valid kernel $\widehat K_n$; we do
not claim that Markovization preserves squared $L^2(\mu\otimes\lambda)$ error.
\end{theorem}
\noindent The proof of Theorem~\ref{thm:markovization} is given in Appendix~\ref{app:markov}.

\subsection{H{\"o}lder--ReLU Rate and Minimax Near-Optimality}
\label{sec:relu}

We now instantiate the oracle theory in a standard nonparametric regime.  The
transition density \(k_0(y\mid x)\) is viewed as a function on
\([0,1]^d\times[0,1]^d\), hence the effective dimension is \(2d\).  Under a
H{\"o}lder smoothness condition and a clipped ReLU approximation--entropy
hypothesis, the anchored contrastive estimator attains the usual nonparametric
rate
\[
    n^{-2\beta/(2\beta+2d)}
\]
up to logarithmic factors.  The same rate transfers to the Markovized kernel in
integrated total variation.  The minimax lower bound below shows that the
squared \(L^2(\mu\otimes\lambda)\) rate is optimal up to logarithmic factors.
Throughout this section, set
\[
    \overline d:=2d .
\]

\begin{assumption}[H{\"o}lder--ReLU approximation and entropy]
\label{ass:holder-relu}
Let $\X=\mathcal Y=[0,1]^d$, and let $\lambda$ be Lebesgue probability measure
on $\mathcal Y$.  For every integer $S\ge2$, there is a clipped ReLU class
$\A(S)$ with values in $[\gamma,\Gamma]$, consisting of jointly measurable
functions on $\X\times\mathcal Y$, such that the following holds.

For every fixed H{\"o}lder radius $R>0$, there exist constants
$C_{\mathrm{app}}(R),C_{\mathrm{ent}}(R)<\infty$ and
$\kappa_{\mathrm{app}},\kappa_{\mathrm{ent}}\ge0$ such that, whenever
$a_0\in\mathcal H_{\overline d}^{\beta}(R)$ and
\[
    \gamma\le\eps r_-,
    \qquad
    \Gamma\ge (1-\eps)M_0+\eps r_+,
\]
we have, for all $S\ge2$ and all $n\ge3$,
\begin{align*}
 \inf_{a\in\A(S)}\norm{a-a_0}_{\infty}
        &\le
        C_{\mathrm{app}}(R)S^{-\beta/\overline d}
        (\log S)^{\kappa_{\mathrm{app}}},\\
 \log N^{\mathrm{int}}_\infty(n^{-1},\A(S))
        &\le
        C_{\mathrm{ent}}(R)S(\log n)^{\kappa_{\mathrm{ent}}}.
\end{align*}
Here $N^{\mathrm{int}}_\infty$ denotes the internal sup-norm covering number.
\end{assumption}
For $\beta>0$, $R>1$, and $0<m<1<M$, let $\K_\beta(R,m,M)$ denote the class of
transition densities $k(y\mid x)$ on $[0,1]^d\times[0,1]^d$ such that
\[
    k\in\mathcal H_{\overline d}^{\beta}(R),
    \qquad
    m\le k(y\mid x)\le M,
    \qquad
    \int_{[0,1]^d} k(y\mid x)\lambda(\dd y)=1
    \quad\text{for every }x.
\]

We now state the concrete rate consequence.  The upper bound is formulated for
the anchored contrastive estimator over the clipped ReLU sieve from
Assumption~\ref{ass:holder-relu}.  The minimax comparison is stated over the
following standard bounded H{\"o}lder class of transition densities.

\begin{theorem}[H{\"o}lder--ReLU Rate and Minimax Near-Optimality]
\label{thm:holder-minimax}
Let $\X=[0,1]^d$ and suppose the bounded reference/target regime of
Assumption~\ref{ass:bounds} holds.  Assume also that
Assumption~\ref{ass:holder-relu} holds for the anchored target density $a_0$.
Let $\widehat a_n$ be a measurable $\delta_{\mathrm{opt}}$-approximate empirical
risk minimizer over the clipped ReLU sieve
\[
    \A_n:=\A(S_n),
    \qquad
    S_n:=\left\lceil n^{\overline d/(2\beta+\overline d)}\right\rceil,
\]
and define the de-anchored and Markovized estimators
\[
\widetilde k_n:=\frac{\widehat a_n-\eps r}{1-\eps},
\qquad
\widehat k_n:=\mathfrak M\widetilde k_n,
\qquad
\widehat K_n(x,\dd y):=\widehat k_n(y\mid x)\lambda(\dd y).
\]
Then there are constants $C,\kappa<\infty$ such that for all $n\ge3$ and $\eta\in(0,1)$, with
probability at least $1-\eta$,
\[
\norm{\widetilde k_n-k_0}_{2,\mu\lambda}^2
\vee
d_{\mu,\TV}(\widehat K_n,K_0)^2
\le
C\left[
    n^{-2\beta/(2\beta+2d)}(\log n)^\kappa
    +\delta_{\mathrm{opt}}
    +\frac{\log(1/\eta)}{n}
\right].
\]
Moreover, when $\mu=\lambda$ is Lebesgue probability measure on $[0,1]^d$, for every $\beta>0$, $R>1$, and $0<m<1<M$, the minimax risk over the
bounded H{\"o}lder transition-density class $\K_{\beta}(R,m,M)$ satisfies
\[
\inf_{\widehat k}
\sup_{K\in\K_{\beta}(R,m,M)}
\E_K\norm{\widehat k-k}_{2,\mu\lambda}^2
\ge
c\,n^{-2\beta/(2\beta+2d)}
\]
for a constant $c>0$.  Hence the squared $L^2(\mu\otimes\lambda)$ rate is
minimax near-optimal up to logarithmic factors.  The squared TV upper bound
follows from the same estimator and the Markovization argument, but no separate
TV minimax lower bound is claimed.
\end{theorem}

\noindent The proof of Theorem~\ref{thm:holder-minimax} is given in Appendix~\ref{app:rate-minimax}.
Assumption~\ref{ass:holder-relu} is a standard ReLU approximation--entropy input rather
than a new neural approximation theorem; see, for example, \citet{yarotsky2017relu}.

\section{Dynamic Transfer}
\label{sec:dynamics}

This section closes the loop
\[
 \text{contrastive risk}
 \;\Rightarrow\;
 \text{anchored density}
 \;\Rightarrow\;
 \text{valid Markov kernel}
 \;\Rightarrow\;
 \text{dynamical error control}.
\]
We give the finite-horizon transfer theorem from learned kernel error to marginal,
path-law, and occupation-measure error under an explicit occupancy coverage condition.
A final proposition shows why such a coverage assumption is unavoidable.

\subsection{From Contrastive Excess Risk to Finite-Horizon Dynamics}

\begin{theorem}[From Contrastive Excess Risk to Finite-Horizon Dynamics]
\label{thm:finite-horizon-transfer-main}
Let $K_0$ be the target kernel and let $\widehat a$ be a learned anchored score with
$\gamma\le \widehat a\le \Gamma$.  Form the valid learned kernel
$\widehat K$ by de-anchoring and Markovization as in Theorem~\ref{thm:markovization}.
Fix an initial law $\xi$ and a horizon $T\in\mathbb N$.  For any two Markov kernels $K$ and $L$, define the pointwise kernel error
\[
  e_{K,L}(x):=\TV(K(x,\cdot),L(x,\cdot)).
\]
For the learned-kernel application, write $e(x):=e_{K_0,\widehat K}(x)$.

For any Markov kernels $K$ and $L$ and any $m\le T$, the occupancy-weighted perturbation bound
(Theorem~\ref{thm:occupancy} in Appendix~\ref{app:dynamics}) gives
\begin{align}
  \TV(\xi K^m,\xi L^m)
  &\le \sum_{s=0}^{m-1}\int e_{K,L}(x)\,(\xi K^s)(\dd x).
  \label{eq:occupancy-marginal-main}
\end{align}
The same right-hand side bounds the total-variation distance between the length-$m$ path laws
and the normalized length-$m$ occupation measures
$\bar\Gamma_m^K:=\frac1m\sum_{s=0}^{m-1}\xi K^s$.

Assume the occupancy coverage condition: for each $s=0,\ldots,T-1$,
\[
  \xi K_0^s \ll \mu,\qquad C_s:=\esssup_{\mu}\frac{\dd(\xi K_0^s)}{\dd\mu}<\infty.
\]
Then for every $m\le T$:
\begin{enumerate}[label=(\roman*),leftmargin=2em]
  \item \textbf{$L^1$-coverage bound.}
  \[
    \TV(\xi K_0^m,\xi\widehat K^m)
    \le \paren{\sum_{s=0}^{m-1} C_s}\, d_{\mu,\TV}(K_0,\widehat K).
  \]
  \item \textbf{$L^2$-coverage bound.}  With
  $C_{\xi,\mu,m}:=\max_{0\le s<m}C_s$,
  \[
    \TV(\xi K_0^m,\xi\widehat K^m)
    \le m\sqrt{C_{\xi,\mu,m}}\,
       \norm{e}_{L^2(\mu)}.
  \]
  \item \textbf{Normalized occupation measures.}  Write the normalized occupation measure
  $\bar\Gamma_m^K:=\frac1m\sum_{s=0}^{m-1}\xi K^s$.  Then
  \[
    \TV(\bar\Gamma_m^{K_0},\bar\Gamma_m^{\widehat K})
    \le \frac{m-1}{2}\sqrt{C_{\xi,\mu,m}}\,\norm{e}_{L^2(\mu)}.
  \]
\end{enumerate}
By Theorem~\ref{thm:markovization},
$e(x)\le\int\abs{\widetilde k_{\widehat a}(y\mid x)-k_0(y\mid x)}\lambda(\dd y)$
pointwise, and Cauchy--Schwarz gives
$\norm{e}_{L^2(\mu)}\le \norm{\widetilde k_{\widehat a}-k_0}_{L^2(\mu\otimes\lambda)}$.
Finally, the $L^2$ calibration in Theorem~\ref{thm:calibration} yields
\[
  \norm{\widetilde k_{\widehat a}-k_0}_{L^2(\mu\otimes\lambda)}^2
  \le \frac{1}{(1-\eps)^2c_-}\Excess(\widehat a).
\]
Combining these steps, contrastive excess risk transfers directly to finite-horizon dynamical error:
\[
  \TV(\xi K_0^m,\xi\widehat K^m)
  \;\le\;
  \frac{m\sqrt{C_{\xi,\mu,m}}}{(1-\eps)\sqrt{c_-}}
  \,\Excess(\widehat a)^{1/2},
\]
and the normalized occupation measures $\bar\Gamma_m^{K_0}$ and $\bar\Gamma_m^{\widehat K}$
satisfy the same inequality with the factor $m$ replaced by $(m-1)/2$, as in part~(iii).
\end{theorem}
\noindent The proof of Theorem~\ref{thm:finite-horizon-transfer-main} is given in Appendix~\ref{app:dynamics}.
\begin{remark}[Invariant law perturbation under contraction]
The main transfer theorem is finite-horizon and occupancy-weighted.  Stationary
perturbation requires stronger assumptions: if one has a uniform kernel error
$d_{\infty,\TV}(K,L)$ and the Dobrushin coefficient
$$\alpha(K):=\sup_{x,x'}\TV(K(x,\cdot),K(x',\cdot))<1,$$
then the standard Dobrushin contraction yields
\[\TV(\pi_K,\pi_L)\le d_{\infty,\TV}(K,L)/(1-\alpha(K)).\]
See Corollary~\ref{cor:stationary} in Appendix~\ref{app:dynamics}.
\end{remark}

\subsection{Trajectory Sampling Through the Same Reconstruction Interface}
\label{sec:trajectory-interface}

The independent-pair oracle theorem isolates the contrastive geometry most cleanly.
For trajectory data, temporal dependence affects the statistical oracle step but not
the reconstruction interface.  We record a conservative thinning-and-coupling
extension.  Its purpose is not to optimize dependent-data rates, but to show that
de-anchoring, Markovization, and finite-horizon perturbation remain unchanged once a
trajectory oracle bound is available.

Let $(X_t)_{t\ge0}$ be a stationary Markov chain with invariant law $\mu$ and
transition kernel $K_0$.  Its absolute-regularity coefficients are
\[
 \beta_X(s)
 :=
 \sup_{t\ge0}
 \beta\!\left(
   \sigma(X_0,\ldots,X_t),
   \sigma(X_{t+s},X_{t+s+1},\ldots)
 \right).
\]
Fix an integer $q\ge2$ and write $M:=\lfloor N/q\rfloor$.  Assume $M\ge1$.
From a trajectory $X_0,\ldots,X_N$, retain only
\[
        (X_{jq-1},X_{jq}),\qquad j=1,\ldots,M.
\]
For each retained transition, draw
$\widetilde Y_j\sim\nu$ and
$Z_{j1},\ldots,Z_{j\tau}\stackrel{\mathrm{i.i.d.}}{\sim}\nu$ independently
of the trajectory and of all other auxiliary variables, and set
\[
 B_j^{(q)}
 :=
 \bigl(X_{jq-1},X_{jq},\widetilde Y_j,Z_{j1},\ldots,Z_{j\tau}\bigr).
\]
The thinned trajectory empirical risk is
\[
 \widehat{\mathcal R}_{M,q}^{\mathrm{tr}}(a)
 :=
 \frac1M\sum_{j=1}^{M}g_a(B_j^{(q)}).
\]

\begin{theorem}[Trajectory sampling through the reconstruction interface]
\label{thm:trajectory-effective-sample}
Suppose Assumptions~\ref{ass:bounds} and \ref{ass:candidate-bounds} hold.
Let $\widehat a_{M,q}^{\mathrm{tr}}$ be a measurable
$\delta_{\mathrm{opt}}$-approximate minimizer of
$\widehat{\mathcal R}_{M,q}^{\mathrm{tr}}$ over $\A_M$.  Let
$\eta\in(0,1)$ and assume
\[
        (M-1)\beta_X(q-1)\le \eta/2 .
\]
For every $\delta>0$ with $\Nint(\delta,\A_M)<\infty$, there are constants
$C_{\mathrm{or}},C_{\mathrm{ker}},C_{\mathrm{dyn}}<\infty$, depending only on
$(\eps,\tau,r_-,r_+,M_0,\gamma,\Gamma)$, such that with probability at least
$1-\eta$,
\[
 \Excess(\widehat a_{M,q}^{\mathrm{tr}})
 \le
 C_{\mathrm{or}}\,
 \mathfrak R_{M,q}(\eta,\delta),
\]
where
\[
 \mathfrak R_{M,q}(\eta,\delta)
 :=
 \inf_{a\in\A_M}\Excess(a)
 +L_{\ell}\delta
 +\delta_{\mathrm{opt}}
 +\frac{\log\{4\Nint(\delta,\A_M)/\eta\}}{M}.
\]
Define
\[
 \widetilde k_{M,q}^{\mathrm{tr}}
 :=
 \frac{\widehat a_{M,q}^{\mathrm{tr}}-\eps r}{1-\eps},
 \qquad
 \widehat k_{M,q}^{\mathrm{tr}}
 :=
 \mathfrak M\widetilde k_{M,q}^{\mathrm{tr}},
\]
and let
$\widehat K_{M,q}^{\mathrm{tr}}(x,\dd y)
 :=\widehat k_{M,q}^{\mathrm{tr}}(y\mid x)\lambda(\dd y)$.
Then
\[
 d_{\mu,\TV}(\widehat K_{M,q}^{\mathrm{tr}},K_0)^2
 \le
 C_{\mathrm{ker}}\,
 \mathfrak R_{M,q}(\eta,\delta).
\]
If the occupancy coverage condition of
Theorem~\ref{thm:finite-horizon-transfer-main} holds up to horizon $T$, then for every
integer $1\le h\le T$,
\[
 \TV\!\left(\xi K_0^h,\xi(\widehat K_{M,q}^{\mathrm{tr}})^h\right)
 \le
 h\sqrt{C_{\xi,\mu,h}C_{\mathrm{dyn}}}\,
 \mathfrak R_{M,q}(\eta,\delta)^{1/2}.
\]
The same right-hand side also bounds the total-variation distance between the
corresponding length-$h$ path laws.  The normalized occupation-measure bound holds with
$h$ replaced by $(h-1)/2$.
\end{theorem}
\noindent The proof of Theorem~\ref{thm:trajectory-effective-sample} is given in
Appendix~\ref{app:trajectory-effective-sample}.

\begin{corollary}[Geometrically mixing trajectories]
\label{cor:geometric-trajectory}
Suppose, in addition to the assumptions of
Theorem~\ref{thm:trajectory-effective-sample}, that
\[
        \beta_X(s)\le B\rho^s,
        \qquad 0<B<\infty,\quad 0<\rho<1.
\]
For
\[
 q_{N,\eta}
 :=
 \max\left\{
   2,\,
   1+\left\lceil\frac{\log(2BN/\eta)}{|\log\rho|}\right\rceil
 \right\},
 \qquad
 N_{\mathrm{eff}}
 :=
 \left\lfloor\frac{N}{q_{N,\eta}}\right\rfloor,
\]
and $N_{\mathrm{eff}}\ge1$, the conclusions of
Theorem~\ref{thm:trajectory-effective-sample} hold with
$q=q_{N,\eta}$ and $M=N_{\mathrm{eff}}$.  In particular,
$N_{\mathrm{eff}}\asymp N/\log(N/\eta)$ as $N/\eta$ grows, up to constants
depending on $(B,\rho)$.
\end{corollary}
\noindent The proof of Corollary~\ref{cor:geometric-trajectory} is given in
Appendix~\ref{app:geometric-trajectory}.

Under the H\"older--ReLU assumptions of Theorem~\ref{thm:holder-minimax}, the same
substitution gives, when the sieve is tuned to the retained sample size,
\[
 d_{\mu,\TV}(\widehat K_{M,q}^{\mathrm{tr}},K_0)^2
 \lesssim
 M^{-2\beta/(2\beta+2d)}(\log M)^\kappa .
\]
For geometrically $\beta$-mixing trajectories with the choice in
Corollary~\ref{cor:geometric-trajectory}, this becomes
\[
 d_{\mu,\TV}(\widehat K_{N_{\mathrm{eff}},q_{N,\eta}}^{\mathrm{tr}},K_0)^2
 \lesssim
 N^{-2\beta/(2\beta+2d)}(\log N)^{\kappa'}
\]
for some $\kappa'>0$.  Thus geometric temporal dependence changes the logarithmic
factor but not the main nonparametric exponent.

\begin{remark}[Statistical anchoring versus restart-controlled data collection]
The geometric $\beta$-mixing assumption is imposed on the observational chain with
kernel $K_0$.  The Doeblin anchor used in the statistical chart does not by itself make
that trajectory geometrically mixing, because the data are sampled from $K_0$, not from
$A_{\eps,\nu}K_0$.

If data collection is instead actively restart-controlled and the observed trajectory is
generated by $A_{\eps,\nu}K_0$, then the Doeblin minorization provides an explicit
geometric mixing mechanism.  This is a different sampling design: the empirical
contrastive construction must be adjusted accordingly, since the observed positive
transitions are already sampled from the anchored kernel.
\end{remark}

\subsection{A limitation example}

\begin{proposition}[Integrated design error need not control stationarity]
\label{prop:limitation}
For every $\rho\in(0,1/2)$ there exist two irreducible Markov kernels $K_\rho$ and
$L_\rho$ on the two-point space $\{0,1\}$ and a design law $\mu_\rho$ such that
\[
        d_{\mu_\rho,\TV}(K_\rho,L_\rho)\le \rho
\]
while their invariant laws satisfy
\[
        \TV(\pi_{K_\rho},\pi_{L_\rho})\ge \frac14.
\]
\end{proposition}
\noindent The proof of Proposition~\ref{prop:limitation} is given in Appendix~\ref{subsec:lim}.

This example shows that the occupancy coverage assumption in
Theorem~\ref{thm:finite-horizon-transfer-main} is not merely a technical convenience: an
integrated learning error under a design law without sufficient coverage is not
automatically a guarantee about stationary behavior.  The obstruction appears already in
a two-point state space.

\section{Two Diagnostic Examples}
\label{sec:examples}

The following finite-state examples use probability masses directly.  Equivalently, the
dominating measure is counting measure on the displayed finite state space, so a kernel
density is just a transition probability mass.

\subsection{A sparse two-state chart}

Let $\X=\{0,1\}$, let $\nu=(1/2,1/2)$, and consider the deterministic flip chain
\[
        K_0(0,\cdot)=\delta_1,
        \qquad
        K_0(1,\cdot)=\delta_0,
        \qquad
        K_0=
        \begin{pmatrix}
          0 & 1\\
          1 & 0
        \end{pmatrix}.
\]
This kernel has zero transition probabilities on half of the state pairs.  The anchored
kernel is
\[
        A_{\eps,\nu}K_0
        =
        \begin{pmatrix}
          \eps/2 & 1-\eps/2\\
          1-\eps/2 & \eps/2
        \end{pmatrix}.
\]
Therefore
\[
        A_{\eps,\nu}K_0(x,\{y\})\ge \eps\nu(\{y\})=\eps/2,
        \qquad x,y\in\{0,1\}.
\]
In mass notation the contrastive risk targets
\[
        a_0(y\mid x)=(1-\eps)k_0(y\mid x)+\eps r(y),
        \qquad r(0)=r(1)=1/2,
\]
and the inverse chart is the entrywise affine map
\[
        k_0(y\mid x)
        =
        \frac{a_0(y\mid x)-\eps r(y)}{1-\eps}.
\]
Thus the original transition may be sparse or deterministic, while the anchored
transition is strictly inside the positive cone determined by the reference law.  The
anchor defines a positive contrastive coordinate whose inverse recovers the original
sparse kernel, at the price of the de-anchoring factor $(1-\eps)^{-1}$.

\subsection{A signed de-anchored row}

The affine inverse of the anchor chart is applied to an estimated score, not to the true
$a_0$.  A de-anchored row can therefore be close in $L^1$ and still fail to be a
probability mass.  For example, suppose the target row at state $0$ is deterministic,
\[
        k_0(0\mid0)=1,\qquad k_0(1\mid0)=0,
\]
but the de-anchored score for that row is
\[
        \widetilde k(0\mid0)=1.01,
        \qquad
        \widetilde k(1\mid0)=-0.01.
\]
The row sum is one and the $L^1$ error on this row is
\[
        |1.01-1|+|-0.01-0|=0.02,
\]
but $\widetilde k(\cdot\mid0)$ is not a probability mass because it has a negative
entry.  Markovization clips and renormalizes:
\[
        (1.01,-0.01)
        \longmapsto
        (1.01,0)
        \longmapsto
        (1,0).
\]
This is the finite-state picture behind Theorem~\ref{thm:markovization}: the learned
score must be returned to the cone of Markov kernels before it is iterated.

\section{Experiments}
\label{sec:experiments}

The experiments are designed as interface diagnostics rather than leaderboard
benchmarks.  The goal is not to claim superiority over specialized conditional density
estimators, but to test whether each operation required by the theory is empirically
necessary: anchoring improves the contrastive coordinate, de-anchoring targets the
original kernel rather than the anchored kernel, Markovization restores transition-kernel
validity after score learning, and one-step kernel error predicts finite-horizon
dynamical error under coverage.  We therefore test the full pipeline
\[
 \widehat a_\theta
 \longrightarrow
 \widetilde k_\theta
 \longrightarrow
 \widehat k_\theta=\mathfrak M\widetilde k_\theta
 \longrightarrow
 d_{\mu,\TV}(\widehat K_\theta,K_0)
 \longrightarrow
 \text{finite-horizon dynamic error}.
\]
Unlike the finite-state diagnostics in Section~\ref{sec:examples}, the main experiments train the
anchored contrastive score.  The estimator is a clipped ReLU network
\[
        a_\theta(x,y)=\gamma+(\Gamma-\gamma)\sigma(f_\theta(x,y)),
\]
trained by binary contrastive logistic loss with positive samples from the anchored
transition $(1-\eps)K_0(x,\cdot)+\eps\nu$ and negative samples from $\nu$.  In the
public implementation this anchored positive law is represented exactly by the weighted
double-positive empirical risk in \eqref{eq:empirical-risk}: for each observed pair the
loss includes both the data transition with weight $1-\eps$ and an independently drawn
restart transition with weight $\eps$.  The network output is the bounded score
$a_\theta$ itself; the contrastive classification logit is then
$\log\{a_\theta(x,y)/(\tau r(y))\}$.  Thus the training objective uses the same bounded
candidate parametrization and weighted empirical risk as the theoretical estimator.
After
training, we de-anchor, Markovize, and evaluate the resulting kernel on common
Monte Carlo or grid integration sets.  The reproducibility repository is available at \url{https://github.com/Ao-Xu/doeblin_exp_2026}. 

\subsection{Protocol, models, and evaluation metrics}

The synthetic models are deliberately chosen to stress different parts of the theory:
smooth kernels test calibration in a regular setting; multimodal and nearly
deterministic kernels test non-Gaussian and sparse transition behavior; trajectory
experiments isolate temporal dependence; and rare-state chains test the necessity of
occupancy coverage.  The complete experiment-to-output index, synthetic model grid,
and implementation parameters are recorded in the repository README for reproduction.
The experiments are synthetic because the target kernel must be known in order to
evaluate integrated TV, path-law TV upper bounds, and coverage failure exactly.  These experiments
should therefore be read as controlled diagnostics of the reconstruction chain rather
than as uncontrolled real-data benchmarks.  All training-based figures in this section
are generated end-to-end by the public script; deterministic diagnostic examples are
identified separately.

\begin{table}[t]
\centering
\footnotesize
\resizebox{\textwidth}{!}{%
\begin{tabular}{p{0.25\textwidth}p{0.30\textwidth}p{0.31\textwidth}}
\toprule
Component & Theory estimator & Public implementation\\
\midrule
Positive samples & Weighted double-positive risk
$(1-\eps)\ell^+(X,Y)+\eps\ell^+(X,\widetilde Y)$ &
The same weighted double-positive construction; no Bernoulli mixture thinning is used.\\
Score parametrization & Bounded candidate functions
$\gamma\le a\le\Gamma$; posterior $a/(a+\tau r)$ & Bounded score network
$a_\theta=\gamma+(\Gamma-\gamma)\sigma(f_\theta)$; logistic logit
$\log\{a_\theta/(\tau r)\}$.\\
Calibration plot & Excess contrastive risk versus density error &
Held-out weighted validation excess
$\widehat R_{\rm val}(\widehat a)-\widehat R_{\rm val}(a_0)$; grid density errors are oracle diagnostics.\\
Reference and Markovization & Restart density with lower envelope; valid-kernel repair &
Matched sampler/density pairs; poor coverage is a $0.1$ uniform plus $0.9$ beta
mixture; continuous-state integrals use grid quadrature.\\
\bottomrule
\end{tabular}}
\caption{Alignment between the theoretical estimator, implemented estimator, and evaluated quantities.}
\label{tab:implementation-alignment}
\end{table}

For continuous kernels, Markovization and TV integrals are evaluated by grid quadrature
on common design and integration points.  We report contrastive excess risk, anchored
and de-anchored squared errors, integrated total variation, and the pre-Markovization
invalidity diagnostics
\[
 \mathrm{NegMass}(\widetilde k_\theta)
  = \E_{X\sim\mu}\int (-\widetilde k_\theta(y\mid X))_+\,d\lambda(y),
 \quad
 \mathrm{RowErr}(\widetilde k_\theta)
  = \E_{X\sim\mu}\left|\int \widetilde k_\theta(y\mid X)\,d\lambda(y)-1\right|.
\]
For finite-state dynamics we compute rollout TV, occupation TV, stationary TV, and
occupancy-weighted upper bounds for path-law TV directly from matrices.  Unless stated
otherwise, Monte Carlo summaries use ten seeds with standard-error bars.  The method
comparison includes the proposed pipeline, nearby ablations, and a row-normalized
Gaussian conditional-density baseline evaluated on the same grids.

\subsection{End-to-end calibration of anchored contrastive learning}

The first experiment tests the calibration implication
\[
 \Excess(a_\theta)
 \quad\text{tracks}\quad
 \norm{a_\theta-a_0}_{2,\mu\lambda}^2
 \quad\text{and}\quad
 \norm{\widetilde k_\theta-k_0}_{2,\mu\lambda}^2 .
\]
We use smooth, multimodal, and rough one-dimensional wrapped transition kernels.  Each
sample size and seed trains a fresh anchored contrastive network.  The excess risk is
evaluated on independent validation contrastive blocks, while density errors are
computed on a fixed integration grid.

\begin{figure}[t]
\centering
\includegraphics[width=\textwidth]{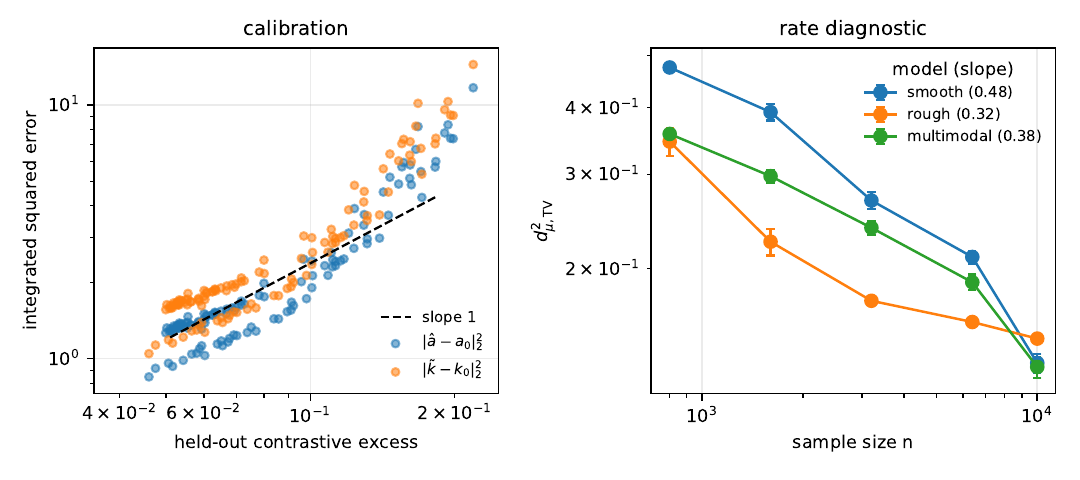}
\caption{Statistical diagnostics for the contrastive front end.  Left: end-to-end
calibration, where each point is a trained neural score under one model, sample size,
and seed.  The log--log cloud aligns with a slope-one reference in aggregate.  Right:
real-trained one-dimensional rate diagnostic, with fitted slopes shown in parentheses.}
\label{fig:exp-calibration-v2}
\end{figure}

Figure~\ref{fig:exp-calibration-v2} shows that validation excess risk orders the trained
scores consistently with the density errors used in the theory: the log--log
correlation with anchored $L^2$ error is about $0.94$, and de-anchoring preserves this
ordering up to the affine factor $(1-\eps)^{-1}$.  We interpret excess risk as a
model-selection proxy, not as a tight finite-sample numerical surrogate for density
loss.

\subsection{Valid-kernel reconstruction and Markovization}

The second experiment checks whether Markovization is actually needed for learned
scores.  We train contrastive scores, de-anchor them, and compare the signed
pre-Markovization score with the Markovized density.  The key statistic is the
empirical repair ratio
\[
        R_{\mathrm M}
        =
        \frac{\norm{\mathfrak M\widetilde k_\theta-k_0}_{1,\mu}}
             {\norm{\widetilde k_\theta-k_0}_{1,\mu}},
\]
which is bounded by $2$ in Theorem~\ref{thm:markovization}.

\begin{figure}[t]
\centering
\includegraphics[width=\textwidth]{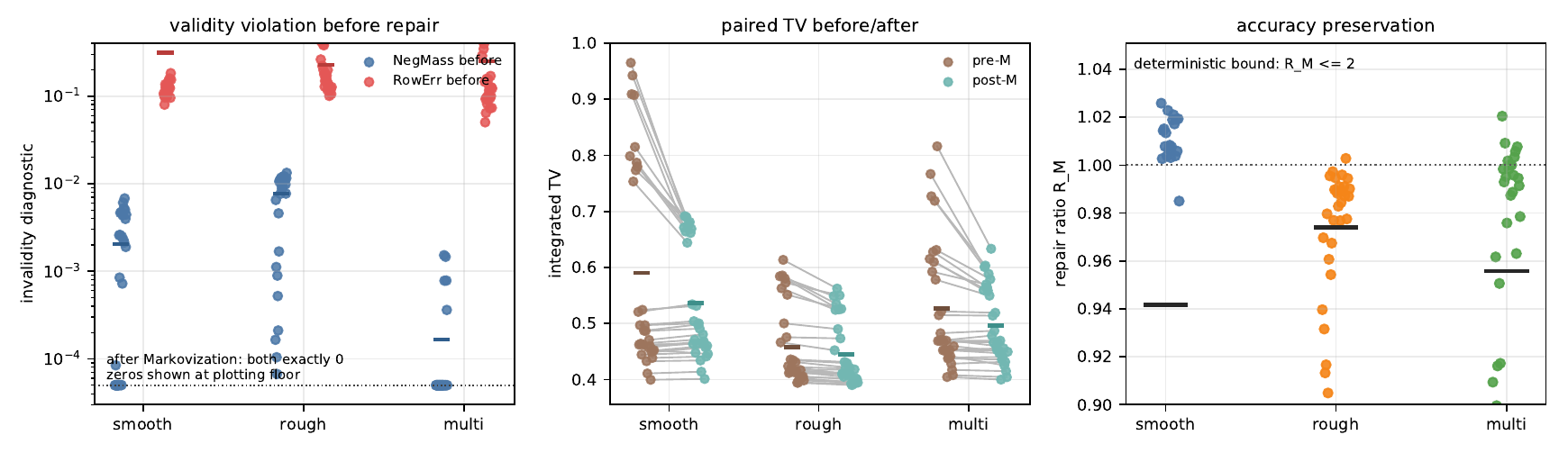}
\caption{Markovization repairs validity without amplifying reconstruction error.
Left: learned de-anchored scores have nonzero negative mass or row-sum error before
Markovization; after Markovization both diagnostics are exactly zero by construction.
Middle: paired pre/post TV values show that Markovization is a validity restoration
step rather than a TV-improvement heuristic.  Right: the empirical repair ratio is
close to one across environments and far below the deterministic bound
$R_{\mathrm M}\le2$.}
\label{fig:exp-markovization-v2}
\end{figure}

Figure~\ref{fig:exp-markovization-v2} shows nonzero negative mass and row error before
post-processing, both exactly zero after Markovization.  The repair ratio is well below
$2$ in all runs, with maximum about $1.03$ and mean about $0.96$.  Thus Markovization is
a validity-restoration step, not a TV-improvement heuristic.

\subsection{Statistical rates, smoothness, and dimension}

The third experiment is a real-trained one-dimensional rate diagnostic.  We do not fill
untrained high-dimensional settings with theory-shaped curves; each sample size and seed
corresponds to a trained contrastive network, evaluated by
$d_{\mu,\TV}(\widehat K_\theta,K_0)^2$.

The right panel of Figure~\ref{fig:exp-calibration-v2} has the qualitative pattern
predicted by the theory: increasing $n$ reduces the reconstruction error.  The fitted
ordinary-least-squares log--log slopes for
$d_{\mu,\TV}(\widehat K_\theta,K_0)^2$ are approximately $0.48$ for the smooth model,
$0.32$ for the rough model, and $0.38$ for the multimodal model.  These slopes should
not be read as precise estimates of the asymptotic exponent or as a high-dimensional
benchmark; they are a sanity check that the real-trained reconstruction error
decreases with sample size.

\subsection{Anchor strength and reference-law coverage}

The fourth experiment is a diagnostic for the two design parameters discussed in
Section~\ref{sec:anchor-choice}; it is not a claim that the paper solves optimal
restart-law selection.  We vary $\eps$ and the reference law, retraining a fresh network
for each setting.  The four reference curves use the fixed order uniform,
poor-coverage, empirical-KDE, and mixture, with matched samplers and density evaluators.
The experiment tests the conditioning--inversion tradeoff in $\eps$ and the effect of
next-state reference coverage.

\begin{figure}[t]
\centering
\includegraphics[width=\textwidth]{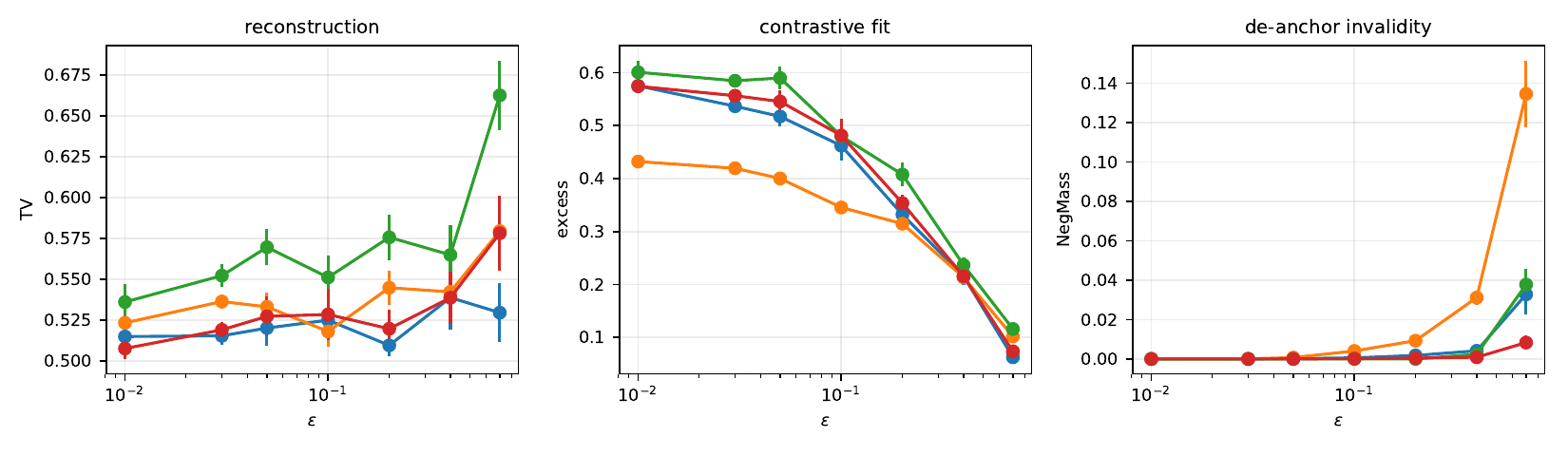}
\caption{Anchor strength and reference-law coverage.  Each $\eps$ and reference law is
trained separately.  Curves use the fixed reference-law order uniform,
poor-coverage, empirical-KDE, and mixture in all three panels.  The reconstruction TV
has an interior optimum; poor reference coverage worsens reconstruction and
invalidity.}
\label{fig:exp-anchor-v2}
\end{figure}

The observed TV curves show the expected conditioning--inversion tradeoff, with a
problem-dependent best reference.  High $\eps$ increases invalidity, especially for the
poor-coverage reference, supporting the interpretation of $\eps$ as a genuine
statistical parameter rather than a harmless technical constant.

\subsection{Trajectory-data stress test}

The preceding oracle theory is stated for independent transition pairs.  We therefore
include a trajectory-only stress test on lazy finite-state chains, comparing adjacent
transition-pair estimates with i.i.d. transition pairs and thinned trajectories.  The
experiment isolates temporal dependence; it is not used to claim sharp dependent-data
rates or to analyze full-trajectory ERM.  In the left and middle panels below, colors
use increasing mixing parameter $\alpha\in\{0.02,0.05,0.1,0.2,0.5\}$.

\begin{figure}[t]
\centering
\includegraphics[width=\textwidth]{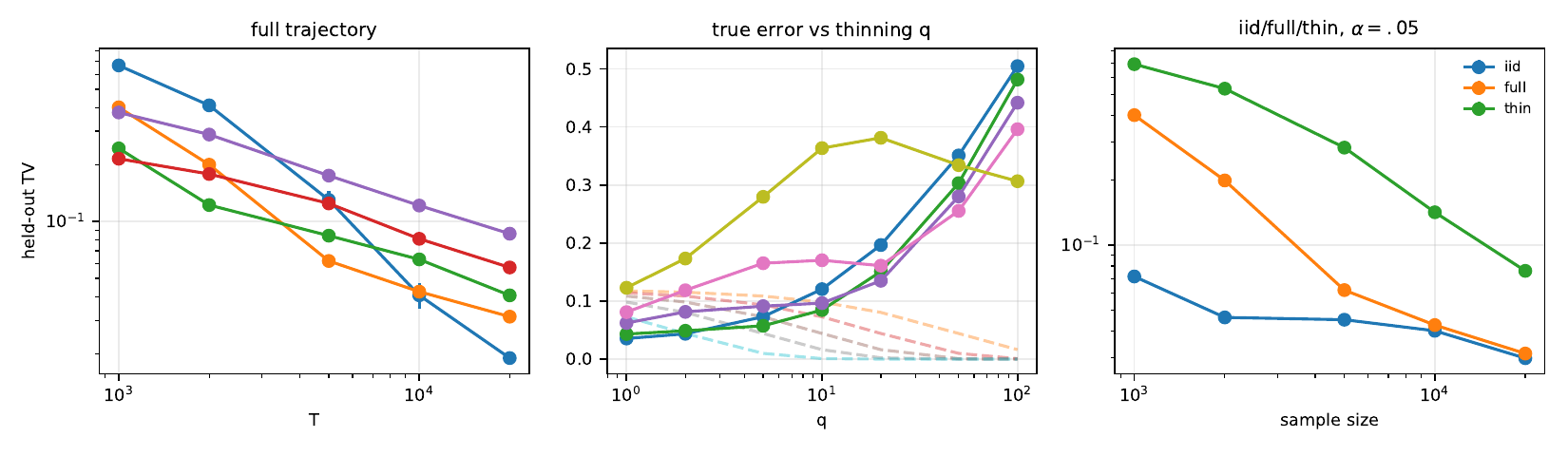}
\caption{Trajectory-data stress test.  Left: held-out TV decreases with trajectory
length but worsens as mixing slows.  Middle: thinning trades sample size for weaker
dependence; colors use the same increasing-$\alpha$ order as the left panel.  Right:
i.i.d. transition pairs are easiest, full trajectories are harder, and thinning can
help in the slow-mixing regime.}
\label{fig:exp-trajectory-v2}
\end{figure}

Figure~\ref{fig:exp-trajectory-v2} shows the expected ordering: fast-mixing trajectories
are closer to independent pairs, slow-mixing trajectories need larger $T$, and thinning
can help in the slow-mixing regime.  The result is consistent with the conservative
thinning theorem in Section~\ref{sec:trajectory-interface}, but is not a sharp
dependent-data rate claim.

\subsection{Finite-horizon dynamical transfer and coverage failure}

The sixth experiment closes the loop from learned one-step kernels to Markov dynamics.
For each trained score we de-anchor, Markovize on a common grid, and roll out the
resulting kernel.  In the covered setting we take an initial law dominated by the design
law and compute
\[
        \TV(\xi \widehat K_\theta^m,\xi K_0^m),
        \qquad m\in\{1,2,5,10,20,50\}.
\]
In the uncovered setting we use the rare-state construction
$\mu(1)=\delta\in\{0.01,0.02,0.05\}$ and start from $\xi=\delta_1$.  The displayed
figure focuses on finite-horizon transfer and the coverage obstruction; stationary
perturbation is already covered by Corollary~\ref{cor:stationary}.

\begin{figure}[t]
\centering
\includegraphics[width=\textwidth]{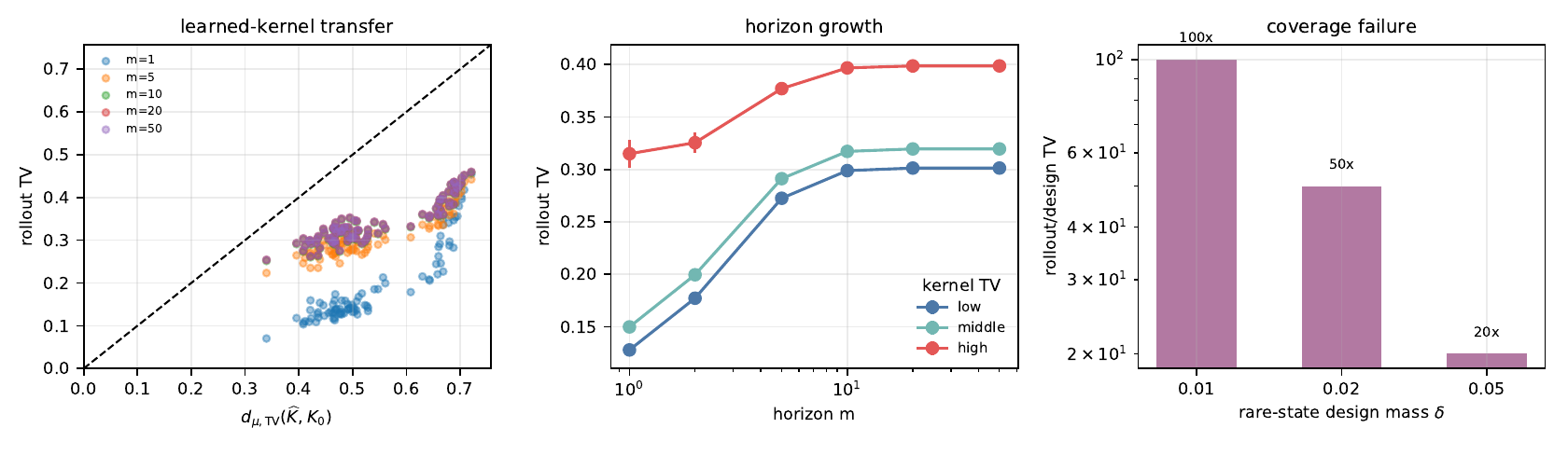}
\caption{Dynamic transfer with learned kernels.  Left: learned-kernel rollout errors
across several horizons are controlled by one-step integrated TV.  Middle: horizon
growth is larger for learned kernels with larger one-step error.  Right: without
coverage, a small design-averaged error can be amplified by factors of $20$--$100$ at
one step.}
\label{fig:exp-dynamics-v2}
\end{figure}

In the covered setting, one-step integrated TV and rollout TV are strongly correlated
(above $0.90$ for every reported horizon and about $0.91$ at horizon $10$).  The
rare-state panel illustrates the obstruction behind Proposition~\ref{prop:limitation}:
the design-averaged error is $\delta$, but the one-step rollout error from the rare
state is $1$, an amplification factor $1/\delta$.

\begin{table}[t]
\centering
\footnotesize
\begin{tabular}{cccc}
\toprule
Rare-state mass $\delta$ & $d_{\mu,\mathrm{TV}}(K,L)$ & rare-state rollout TV & amplification\\
\midrule
0.01 & 0.01 & 1.00 & $100\times$\\
0.02 & 0.02 & 1.00 & $50\times$\\
0.05 & 0.05 & 1.00 & $20\times$\\
\bottomrule
\end{tabular}
\caption{Coverage-failure diagnostic.  A design-averaged one-step error can be made small by assigning small design mass to a rare state, while rollout from that state remains maximally wrong.  This table visualizes the obstruction in Proposition~\ref{prop:limitation}.}
\label{tab:coverage-failure}
\end{table}

\subsection{Ablation study}

The seventh experiment removes one component at a time: anchoring, de-anchoring,
Markovization, reference coverage, anchor strength, model capacity, and negative-sample
count.  The metrics are the same as above, so the ablation tests whether each interface
component has an observable statistical or dynamical effect.

\begin{figure}[t]
\centering
\includegraphics[width=\textwidth]{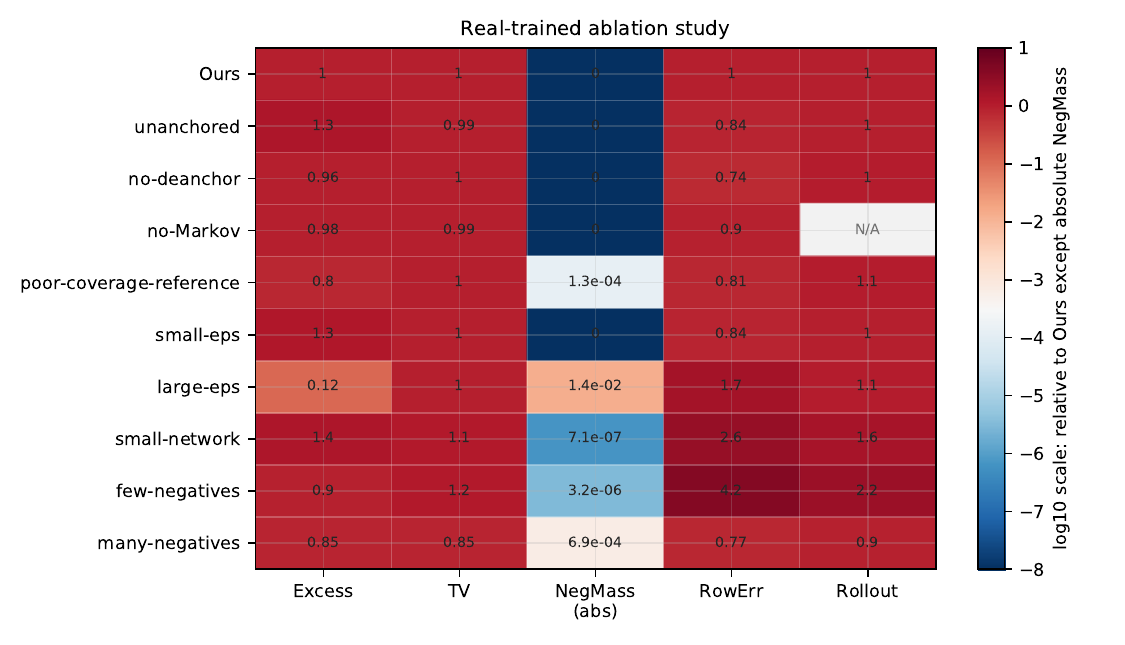}
\caption{Ablation heatmap.  Excess, TV, RowErr, and rollout entries are relative to
Ours on a $\log_{10}$ scale.  The NegMass column instead reports absolute
pre-Markovization negative mass, because Ours has zero NegMass to numerical precision;
a $10^{-8}$ floor is used only for coloring zeros, not as a denominator.  Removing
anchoring, de-anchoring, Markovization, using an under-capacity network, or degrading
reference coverage harms at least one of excess risk, TV reconstruction, invalidity,
or rollout error.}
\label{fig:exp-ablation-v2}
\end{figure}

Table~\ref{tab:method-comparison} compares the full valid-kernel pipeline with nearby
interface variants and a row-normalized Gaussian-CDE baseline on the multimodal
diagnostic.  These comparisons are interface diagnostics, not benchmark claims: the
full pipeline targets the original kernel, repairs validity after de-anchoring, and
produces a kernel suitable for rollout.

\begin{table}[t]
\centering
\footnotesize
\resizebox{\textwidth}{!}{%
\begin{tabular}{llcccc}
\toprule
Method & Role & Numerically valid? & Covered by chart theory? & $d_{\mu,\mathrm{TV}}$ & rollout TV\\
\midrule
Ours & full interface & Yes & Yes & 0.493 & 0.133\\
Ours-noMarkov & score without repair & No & No & 0.489 & N/A\\
Anchored-noDeanchor & wrong target kernel & Yes & No: targets $A_{\eps,\nu}K_0$ & 0.495 & 0.14\\
Unanchored-NCE & accurate but no restart chart & Yes, after Markovization & No & 0.489 & 0.136\\
Gaussian-CDE & row-normalized baseline & Yes & No & 0.514 & 0.505\\
\bottomrule
\end{tabular}
}
\caption{Interface-oriented method comparison after end-to-end training.  Numerical validity records whether the evaluated object is a transition kernel; theory coverage records whether the estimator is covered by the anchored-chart reconstruction theory.  Gaussian-CDE is a row-normalized circular Gaussian regression baseline.  No-Markov is not rolled out because it is not a valid transition kernel.}
\label{tab:method-comparison}
\end{table}

\begin{table}[t]
\centering
\footnotesize
\resizebox{\textwidth}{!}{%
\begin{tabular}{lccccccc}
\toprule
Variant & Excess & TV & Pre-M NegMass & Pre-M RowErr & Num. valid? & Theory? & RolloutTV\\
\midrule
Ours & 0.427 & 0.493 & 0 & 0.13 & Yes & Yes & 0.133\\
unanchored & 0.571 & 0.489 & 0 & 0.109 & Yes & No & 0.136\\
no-deanchor & 0.411 & 0.495 & N/A & N/A & Yes & No: wrong target & 0.14\\
no-Markov & 0.416 & 0.489 & 0 & 0.117 & No & No & N/A\\
poor-coverage-reference & 0.341 & 0.5 & 0.000134 & 0.105 & Yes & stress test & 0.149\\
small-eps & 0.555 & 0.493 & 0 & 0.11 & Yes & Yes & 0.136\\
large-eps & 0.0523 & 0.495 & 0.0138 & 0.226 & Yes & Yes & 0.145\\
small-network & 0.599 & 0.556 & 7.14e-07 & 0.336 & Yes & Yes & 0.215\\
few-negatives & 0.386 & 0.579 & 3.23e-06 & 0.542 & Yes & Yes & 0.299\\
many-negatives & 0.364 & 0.419 & 0.000691 & 0.1 & Yes & Yes & 0.12\\
\bottomrule
\end{tabular}
}
\caption{Ablation summary from real-trained variants.  Pre-M NegMass and Pre-M RowErr refer only to the de-anchored score before Markovization.  N/A means that the corresponding pre-Markovization diagnostic or rollout is not defined for that ablation.  Ours has zero Pre-M NegMass to numerical precision in this run.  Numerical validity is separated from coverage by the anchored-chart theory.  The no-Markov row is not rolled out because it is not a valid transition kernel.}
\label{tab:ablation-summary}
\end{table}

The ablation confirms the intended division of labor.  The invalidity columns are
pre-Markovization diagnostics, so they do not contradict
Theorem~\ref{thm:markovization}; no-deanchor is valid but targets
$A_{\eps,\nu}K_0$ rather than $K_0$; and no-Markov is not rolled out because it is not a
valid kernel.  The anchor is not presented as a universal finite-sample accuracy boost,
but as the structural device that supplies the restart chart and lower envelope.

\subsection{Runtime diagnostics}

The final experiment gives measured runtime diagnostics for data construction,
training, Markovization, and evaluation.  No large-scale timing point is extrapolated.

\begin{figure}[t]
\centering
\includegraphics[width=\textwidth]{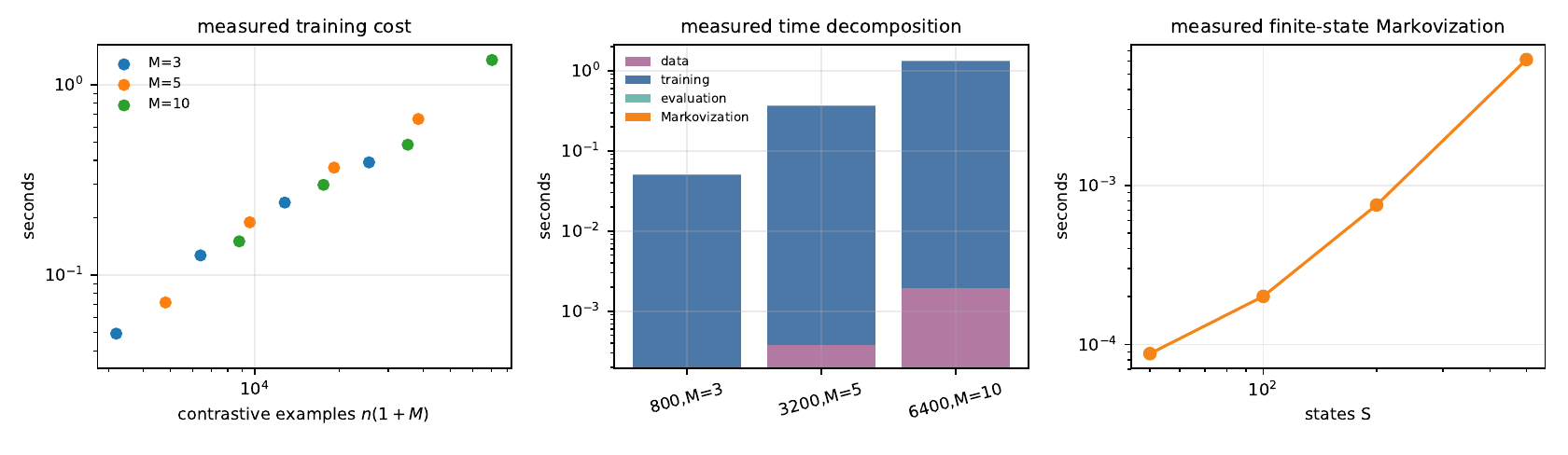}
\caption{Runtime diagnostics.  Left: measured training time versus the number of
contrastive examples $n(1+M)$.  Middle: representative measured wall-clock
decomposition.  Right: measured finite-state Markovization time.}
\label{fig:exp-runtime-v2}
\end{figure}

Figure~\ref{fig:exp-runtime-v2} shows the limited diagnostic we need: in these runs,
the dominant cost is training the contrastive score network, while de-anchoring and
Markovization are small post-processing costs.  This is not a large-scale systems
benchmark, and we do not make broader scalability claims.

\section{Conclusion, Limitations, and Outlook}
\label{sec:conclusion}

This paper introduced a Doeblin-anchored contrastive chart for learning Markov transition
kernels.  The chart organizes the learning problem through a sequence of quantitative
interfaces:
\[
\begin{aligned}
\text{contrastive risk}
&\Longrightarrow
\text{anchored density}
\\
&\Longrightarrow
\text{valid Markov kernel}
\Longrightarrow
\text{finite-horizon dynamics}.
\end{aligned}
\]
Given a restart law and an anchor strength, the anchored transition is simultaneously a
Doeblin-minorized Markov kernel, the positive conditional law in a binary contrastive
experiment, and an explicitly invertible coordinate for the original transition law.

Within this chart, the anchored contrastive risk identifies the anchored transition
density and calibrates excess risk to density error.  De-anchoring recovers a score for
the original transition law, while Markovization restores nonnegativity and row
normalization with only a constant-factor loss in integrated $L^1$ accuracy.  Standard
oracle inequalities and H\"older--ReLU approximation bounds pass through the same
interface.  Under explicit occupancy coverage, valid-kernel error further controls
finite-horizon marginal, path-law, and occupation-measure errors.  For stationary
geometrically $\beta$-mixing trajectories, a conservative thinning-and-coupling argument
preserves the reconstruction interface with an effective sample size.  The experiments
serve as diagnostics of these interfaces rather than as leaderboard benchmarks.

The present analysis has several limitations.  The main oracle theory is stated first
for independent transition pairs, and the trajectory extension analyzes a thinned ERM
rather than an estimator trained on every adjacent transition pair.  The restart law must
be simulable and sufficiently well supported; anchoring makes this role explicit but does
not remove the reference-selection problem.  The bounded-density assumptions provide
uniform curvature and concentration, but exclude heavier-tailed settings.  Finally, the
main dynamical guarantee is finite-horizon and occupancy-weighted.  It should not be
read as a general stationary-distribution perturbation theorem.

Several extensions are natural.  Sharper dependent-data analysis could treat
full-trajectory ERM, nonstationary starts, and optimized mixing-rate dependence.
Adaptive or data-dependent restart laws could improve statistical conditioning and,
under restart-controlled data collection or additional design assumptions, downstream
coverage.  Tail-localized arguments could relax the global boundedness assumptions,
while drift--minorization and nonuniform ergodicity techniques could extend the
dynamical analysis beyond finite horizons.  These directions can be pursued within the
same contrastive-to-valid-kernel interface.

\acks{This work is supported by the Zhongguancun Academy, Grant No.~C20250201.}

\appendix

\section{Proofs and Technical Extensions}
\label{app:deferred-proofs}

This appendix contains the proofs deferred from the main text and the auxiliary
lemmas used in those proofs.
\subsection{TV measurability}
\label{app:tv}
\begin{lemma}[Measurability of the TV kernel error]
\label{lem:tv-measurability}
Let $(\mathcal X,\mathcal B_{\mathcal X})$ be a measurable space and
$(\mathcal Y,\mathcal B_{\mathcal Y})$ be a standard Borel space.  Let $K,L$ be Markov
kernels from $\mathcal X$ to $\mathcal Y$.  Then the map
\[
x\mapsto\TV(K(x,\cdot),L(x,\cdot))
\]
is measurable.
\end{lemma}

\begin{proof}
On a standard Borel space, the \(\sigma\)-algebra \(\mathcal B_{\mathcal Y}\) is countably generated.
Let \(\mathcal C\subset \mathcal B_{\mathcal Y}\) be a countable algebra generating \(\mathcal B_{\mathcal Y}\).  For any two probability
measures \(P,Q\) on \(\mathcal Y\), the finite measure \(P+Q\) allows every
\(B\in\mathcal B_{\mathcal Y}\) to be approximated in \((P+Q)\)-measure by sets in
\(\mathcal C\). Hence
\[
  \TV(K(x,\cdot),L(x,\cdot))
  =\sup_{C\in\mathcal C}\abs{K(x,C)-L(x,C)},
\]
which is a countable supremum of measurable functions of $x$ and hence measurable.
\end{proof}

\subsection{Proof of Theorem~\ref{thm:chart}}
\label{app:chart}
\textit{\underline{Part (i).}}
Assume $0<\eps<1$.  If $K$ is a Markov kernel, then for every $x$ and $B$,
\[
  A_{\eps,\nu}K(x,B)=(1-\eps)K(x,B)+\eps\nu(B)
  \ge \eps\nu(B),
\]
so $A_{\eps,\nu}K$ lies in the stated image class.  Conversely, if $A$ is a Markov
kernel with $A(x,B)\ge\eps\nu(B)$, then $A(x,B)-\eps\nu(B)\ge0$ for all $x$ and $B$.
For each fixed $x$, define
\[
  K_x(B):=\frac{A(x,B)-\eps\nu(B)}{1-\eps}.
\]
Since $A(x,\cdot)-\eps\nu(\cdot)$ is a nonnegative finite measure and
\[
  K_x(\X)=\frac{A(x,\X)-\eps\nu(\X)}{1-\eps}=1,
\]
$K_x$ is a probability measure.  Moreover, for each $B$, the map $x\mapsto K_x(B)$
is measurable, so $K$ is a Markov kernel.  The displayed inverse formula is verified by substituting:
\[
  A_{\eps,\nu}^{-1}(A_{\eps,\nu}K)(x,B)
  =\frac{(1-\eps)K(x,B)+\eps\nu(B)-\eps\nu(B)}{1-\eps}
  =K(x,B).
\]

\noindent\textit{\underline{Part (ii).}}
For every design law $\mu$,
\begin{align*}
 d_{\mu,\TV}(A_{\eps,\nu}K,A_{\eps,\nu}L)
 &=\int\TV\paren{A_{\eps,\nu}K(x,\cdot),A_{\eps,\nu}L(x,\cdot)}\,\mu(\dd x)\\
 &=\int\TV\paren{(1-\eps)K(x,\cdot)+\eps\nu,\,(1-\eps)L(x,\cdot)+\eps\nu}\,\mu(\dd x).
\end{align*}
Since the signed difference satisfies
\[
\big((1-\eps)K(x,\cdot)+\eps\nu\big)-\big((1-\eps)L(x,\cdot)+\eps\nu\big)
=(1-\eps)\big(K(x,\cdot)-L(x,\cdot)\big),
\]
the total variation norm is positively homogeneous, and hence
\[
\TV((1-\eps)K(x,\cdot)+\eps\nu,\,(1-\eps)L(x,\cdot)+\eps\nu)
=(1-\eps)\TV(K(x,\cdot),L(x,\cdot)).
\]
Integrating gives $d_{\mu,\TV}(A_{\eps,\nu}K,A_{\eps,\nu}L)=(1-\eps)d_{\mu,\TV}(K,L)$.
The uniform version is analogous.

\noindent \textit{\underline{Part (iii).}}
If $r(y)>0$ and $a(y\mid x)>0$, then $\eta_{A,a}(x,y)=a(y\mid x)/\{a(y\mid x)+\tau r(y)\}$
is well-defined and satisfies $0<\eta_{A,a}<1$.  Solving for $a$ gives
$a=\tau r\eta/(1-\eta)$.  Substituting into the inverse anchor map and simplifying gives
the de-anchoring formula.  The almost-everywhere qualifier follows because all
identities hold for $\mu\otimes\lambda$-almost every $(x,y)$ when $r>0$ $\lambda$-a.e.

\subsection{Proof of Theorem~\ref{thm:calibration}}
\label{app:risk}
\noindent\textit{\underline{Part (i)}: Identification.}
Fix $(x,y)$ and abbreviate $p=a_0(y\mid x)$, $q=\tau r(y)$, $s=a(y\mid x)$.
The pointwise integrand of \eqref{eq:population-risk-integral}, up to an
$s$-independent term, equals
\[
  F_{p,q}(s):=-p\log s+(p+q)\log(s+q).
\]
One checks that
\[
  F'_{p,q}(s)=-\frac{p}{s}+\frac{p+q}{s+q}
  =\frac{q(s-p)}{s(s+q)}.
\]
Hence $s=p$ is the unique minimizer.  Since $F_{p,q}(s)-F_{p,q}(p)\ge 0$ pointwise
and equality holds if and only if $s=p$, integration implies that the population risk
is uniquely minimized by $a=a_0$, up to $\mu\otimes\lambda$-a.e.\ equality.

\noindent\textit{\underline{Part (ii)}: $L^2$ calibration.}
By Assumption~\ref{ass:bounds}, $a_0=(1-\eps)k_0+\eps r$ satisfies
$\eps r_-\le a_0(y\mid x)\le (1-\eps)M_0+\eps r_+$.  By Assumption~\ref{ass:candidate-bounds},
every $a\in\A_n$ lies in $[\gamma,\Gamma]$.  Hence the triples
$(p,q,s)$ with $p=a_0(y\mid x)$, $q=\tau r(y)$, and $s=a(y\mid x)$ lie in a compact
subset of $(0,\infty)^3$ bounded away from zero.  Define
\[
G(p,q,s)
:=
\begin{cases}
\dfrac{F_{p,q}(s)-F_{p,q}(p)}{(s-p)^2}, & s\ne p,\\[1ex]
\dfrac12 F''_{p,q}(p), & s=p.
\end{cases}
\]
Since $s=p$ is the unique minimizer of $F_{p,q}$, the numerator is strictly
positive whenever $s\ne p$.  The above definition extends continuously to $s=p$,
and $F''_{p,q}(p)>0$ on the admissible compact set.  Hence $G$ is continuous and
strictly positive on the compact admissible set.  Therefore there exist constants
$0<c_-\le c_+<\infty$, depending only on the boundedness constants, such that
\[
c_-(s-p)^2
\le
F_{p,q}(s)-F_{p,q}(p)
\le
c_+(s-p)^2 .
\]
Integrating gives the $L^2(\mu\otimes\lambda)$ calibration.

\noindent\textit{\underline{Part (iii)}: Contrastive KL representation.}
A direct algebraic computation shows
\begin{align*}
 &F_{p,q}(s)-F_{p,q}(p)\\
 &=p\log\frac{p}{s}+(p+q)\log\frac{s+q}{p+q}\\
 &=p\log\frac{p/(p+q)}{s/(s+q)}
  +q\log\frac{q/(p+q)}{q/(s+q)}\\
 &=(p+q)\KL\!\left(
\Ber\!\left(\frac{p}{p+q}\right)
\middle\|
\Ber\!\left(\frac{s}{s+q}\right)
\right),
\end{align*}
which is the integrand in \eqref{eq:contrastive-kl} after substituting
$p=a_0$, $q=\tau r$, $s=a$.

\noindent\textit{\underline{Part (iv)}: $L^1$ and Hellinger consequences.}
Assuming \(\lambda(\mathcal Y)=1\), the product measure \(\mu\otimes\lambda\) is a probability measure.  By Cauchy--Schwarz on \(\mu\otimes\lambda\),
\[
  \norm{a-a_0}_{1,\mu\lambda}^2\le\norm{a-a_0}_{2,\mu\lambda}^2
  \le c_-^{-1}\Excess(a).
\]
For Hellinger, since $a_0\ge\eps r_-$ and $a\ge\gamma$, both functions are bounded
below by $m_a=\min\{\eps r_-,\gamma\}>0$.  The pointwise bound
\[
  (\sqrt{a}-\sqrt{a_0})^2
  =\frac{(a-a_0)^2}{(\sqrt{a}+\sqrt{a_0})^2}
  \le\frac{(a-a_0)^2}{4m_a}
\]
integrates to the Hellinger calibration.  If $a$ is itself a transition density and $A_a(x,\dd y)=a(y\mid x)\lambda(\dd y)$, then
\[
d_{\mu,\TV}(A_a,A_{\eps,\nu}K_0)
=
\frac12\|a-a_0\|_{1,\mu\lambda}.
\]
Therefore
\[
d_{\mu,\TV}(A_a,A_{\eps,\nu}K_0)^2
\le
\frac14\|a-a_0\|_{1,\mu\lambda}^2
\le
\frac{1}{4c_-}\Excess(a).
\]

\subsection{Proof of Theorem~\ref{thm:oracle}}\label{app:erm}
Before presenting the complete proof of Theorem~\ref{thm:oracle}, we first introduce the following two lemmas.
\begin{lemma}
\label{lem:loss-cover}
Under Assumptions~\ref{ass:bounds} and \ref{ass:candidate-bounds}, there is a constant
$L_{\ell}<\infty$, depending only on $(\tau,r_-,r_+,\gamma,\Gamma)$, such that
for all $a,\bar a\in\A_n$ with $\norm{a-\bar a}_\infty\le\delta$,
\[
  \abs{\mathcal R(a)-\mathcal R(\bar a)}\le L_{\ell}\delta,
  \qquad
  \abs{\widehat{\mathcal R}_n(a)-\widehat{\mathcal R}_n(\bar a)}\le L_{\ell}\delta.
\]
\end{lemma}

\begin{proof}
Recall that for any candidate $a\in\A_n$ and fixed $(x,y)$, the positive and negative loss components are
\[
  \ell_a^+(x,y) = -\log\frac{a(y\mid x)}{a(y\mid x)+\tau r(y)}, \qquad
  \ell_a^-(x,y) = -\log\frac{\tau r(y)}{a(y\mid x)+\tau r(y)}.
\]
For a fixed $(x,y)$, let $s = a(y\mid x)$ and $q = \tau r(y)$. Under Assumptions~\ref{ass:bounds} and \ref{ass:candidate-bounds}, we have $s \in [\gamma, \Gamma]$ and $q \in [\tau r_-, \tau r_+]$. Viewed as functions of $s$, the loss components are $f^+(s) := -\log(s) + \log(s+q)$ and $f^-(s) := -\log(q) + \log(s+q)$. 

Their derivatives with respect to $s$ are given by
\[
  \abs{\frac{\dd}{\dd s} f^+(s)} 
  = \abs{-\frac{1}{s} + \frac{1}{s+q}} 
  = \frac{q}{s(s+q)}, 
  \qquad
  \abs{\frac{\dd}{\dd s} f^-(s)} 
  = \frac{1}{s+q}.
\]
We can uniformly bound these derivatives on the admissible domain $[\gamma, \Gamma] \times [\tau r_-, \tau r_+]$:
\[
  \abs{\frac{\dd}{\dd s} f^+(s)} \le \frac{\tau r_+}{\gamma(\gamma + \tau r_-)} =: L^+,
  \qquad
  \abs{\frac{\dd}{\dd s} f^-(s)} \le \frac{1}{\gamma + \tau r_-} =: L^-.
\]
By the Mean Value Theorem, the loss functions are Lipschitz continuous in the candidate score $a(y\mid x)$. For any $a, \bar a \in \A_n$ with $\norm{a-\bar a}_\infty \le \delta$, it follows pointwise that
\[
  \abs{\ell_a^+(x,y) - \ell_{\bar a}^+(x,y)} \le L^+\delta, \qquad
  \abs{\ell_a^-(x,y) - \ell_{\bar a}^-(x,y)} \le L^-\delta.
\]
The empirical risk $\widehat{\mathcal R}_n$ is a linear combination of these loss terms. By the triangle inequality, for any realization of the augmented dataset,
\begin{align*}
  \abs{\widehat{\mathcal R}_n(a) - \widehat{\mathcal R}_n(\bar a)} 
  &\le \frac{1}{n}\sum_{i=1}^n \Bigg[ (1-\eps)\abs{\ell_a^+(X_i,Y_i) - \ell_{\bar a}^+(X_i,Y_i)} \\
  &\quad + \eps\abs{\ell_a^+(X_i,\widetilde Y_i) - \ell_{\bar a}^+(X_i,\widetilde Y_i)} + \sum_{j=1}^\tau \abs{\ell_a^-(X_i,Z_{ij}) - \ell_{\bar a}^-(X_i,Z_{ij})} \Bigg] \\
  &\le (1-\eps)L^+\delta + \eps L^+\delta + \tau L^-\delta \\
  &= (L^+ + \tau L^-)\delta.
\end{align*}
Setting $L_{\ell} := L^+ + \tau L^-$ yields the uniform bound for the empirical risk. Since this pointwise bound holds surely, integrating with respect to the population distribution immediately yields the identical bound $\abs{\mathcal R(a) - \mathcal R(\bar a)} \le L_{\ell}\delta$ for the population risk.
\end{proof}

\begin{lemma}[Entropy bound on the uniform deviation]
\label{prop:entropy}
Under Assumptions~\ref{ass:bounds} and \ref{ass:candidate-bounds}, there exists a constant
\(B_\ell<\infty\), depending only on \((\tau,r_-,r_+,\gamma,\Gamma)\), such that, with
\[
  \Delta_n:=\E\sup_{a\in\A_n}
  \abs{\widehat{\mathcal R}_n(a)-\mathcal R(a)},
  \]
for every \(\delta>0\),
\[
  \Delta_n
  \le
  2L_{\ell}\delta
  +
  B_{\ell}
  \sqrt{\frac{2\log(2N^{\mathrm{int}}_\infty(\delta,\mathcal A_n))}{n}} .
\]
\end{lemma}

\begin{proof}
Since \(a\in[\gamma,\Gamma]\) and \(r\in[r_-,r_+]\), both \(\ell_a^+\) and \(\ell_a^-\) are uniformly bounded.
Moreover, the one-block loss \(g_a(B)\) is a fixed finite linear combination of these bounded terms,
so there exists a constant \(B_\ell<\infty\) such that \(|g_a(B)|\le B_\ell\) uniformly over \(a\in\A_n\).

Let $\{a_1,\ldots,a_N\}$ be a $\delta$-cover of $\A_n$ in the sup-norm, where $N=N^{\mathrm{int}}_\infty(\delta,\mathcal A_n)$.
For any arbitrary $a \in \A_n$, we can select a cover point $a_j$ such that $\norm{a-a_j}_\infty \le \delta$. 
By the bounded loss and Lipschitz covering reduction (Lemma~\ref{lem:loss-cover}), we have both
\[
  \abs{\widehat{\mathcal R}_n(a)-\widehat{\mathcal R}_n(a_j)} \le L_{\ell}\delta 
  \quad \text{and} \quad 
  \abs{\mathcal R(a)-\mathcal R(a_j)} \le L_{\ell}\delta.
\]
Applying the triangle inequality, the uniform deviation can be bounded by the maximum deviation over the finite cover:
\begin{align*}
  \sup_{a\in\A_n}\abs{\widehat{\mathcal R}_n(a)-\mathcal R(a)}
  &\le \sup_{a\in\A_n} \left( \abs{\widehat{\mathcal R}_n(a)-\widehat{\mathcal R}_n(a_j)} + \abs{\widehat{\mathcal R}_n(a_j)-\mathcal R(a_j)} + \abs{\mathcal R(a_j)-\mathcal R(a)} \right) \\
  &\le \max_{j=1,\ldots,N}\abs{\widehat{\mathcal R}_n(a_j)-\mathcal R(a_j)} + 2L_{\ell}\delta.
\end{align*}
Taking expectations on both sides, it remains to bound the expected maximum of the empirical processes. 
Since \(|g_{a_j}(B)|\le B_\ell\), the centered variables
\(g_{a_j}(B)-\mathbb E g_{a_j}(B)\) are uniformly bounded by \(2B_\ell\).
Enlarging \(B_\ell\) if necessary, Hoeffding's lemma implies that
\(\widehat{\mathcal R}_n(a_j)-\mathcal R(a_j)\) is sub-Gaussian with variance proxy
of order \(B_\ell^2/n\).
Since $\abs{Z} = \max\{Z, -Z\}$, the maximum of the absolute values corresponds to the maximum of $2N$ sub-Gaussian random variables. By the standard maximal inequality for finite sub-Gaussian classes (e.g., Massart's lemma \citep{boucheron2013concentration}), we obtain
\[
  \E \left[ \max_{j=1,\ldots,N}\abs{\widehat{\mathcal R}_n(a_j)-\mathcal R(a_j)} \right] 
  \le B_{\ell}\sqrt{\frac{2\log(2N)}{n}}.
\]
Adding the $2L_{\ell}\delta$ deterministic approximation error yields the stated entropy bound.
\end{proof}

\noindent We now present the complete proof of Theorem~\ref{thm:oracle}.

\noindent\textit{\underline{Part (i).}}
Fix any $a^\star\in\A_n$.  Since $\widehat a_n$ is
$\delta_{\mathrm{opt}}$-optimal for the empirical risk,
\[
   \widehat{\mathcal R}_n(\widehat a_n)
      \le \widehat{\mathcal R}_n(a^\star)+\delta_{\mathrm{opt}}.
\]
Add and subtract population risks to obtain
\begin{align*}
 \mathcal R(\widehat a_n)-\mathcal R(a_0)
 &\le
   \bracks{\mathcal R(\widehat a_n)-\widehat{\mathcal R}_n(\widehat a_n)}
   +\delta_{\mathrm{opt}}
   +\bracks{\widehat{\mathcal R}_n(a^\star)-\mathcal R(a^\star)}
   +\bracks{\mathcal R(a^\star)-\mathcal R(a_0)} \\
 &\le
   2\sup_{a\in\A_n}\abs{\widehat{\mathcal R}_n(a)-\mathcal R(a)}
   +\delta_{\mathrm{opt}}
   +\bracks{\mathcal R(a^\star)-\mathcal R(a_0)}.
\end{align*}
Theorem~\ref{thm:calibration}(ii) and the upper calibration bound give
\[
 \norm{\widetilde k_{\widehat a_n}-k_0}_{2,\mu\lambda}^2
 \le \frac{2}{(1-\eps)^2c_-}
      \sup_{a\in\A_n}\abs{\widehat{\mathcal R}_n(a)-\mathcal R(a)}
    +\frac{\delta_{\mathrm{opt}}}{(1-\eps)^2c_-}
    +\frac{c_+}{(1-\eps)^2c_-}
      \norm{a^\star-a_0}_{2,\mu\lambda}^2.
\]
Taking expectations on both sides and taking the infimum over $a^\star \in \A_n$, we obtain, by the definition of the expected uniform deviation 
\[
  \Delta_n:=\E\sup_{a\in\A_n} \abs{\widehat{\mathcal R}_n(a)-\mathcal R(a)},
\]
the following intermediate bound:
\[         
\E\norm{\widetilde k_{\widehat a_n}-k_0}_{2,\mu\lambda}^2         
\le              
\frac{2\Delta_n+\delta_{\mathrm{opt}}}{(1-\eps)^2c_-}                
+\frac{c_+}{(1-\eps)^2c_-}                     
\inf_{a\in\A_n}\norm{a-a_0}_{2,\mu\lambda}^2.      
\]  
Finally, applying Lemma~\ref{prop:entropy} to upper-bound $\Delta_n$ yields the stated entropy bound.

\noindent\textit{\underline{Part (ii).}}
The blocks are i.i.d. by construction.  By Assumption~\ref{ass:bounds}, $a_0$ is
bounded above and below by positive constants, and by Assumption~\ref{ass:candidate-bounds},
every $a\in\A_n$ lies in $[\gamma,\Gamma]$, while
\[
    \eps r_- \le a_0(y\mid x)
    \le (1-\eps)M_0+\eps r_+
\]
almost everywhere. Hence both \(a\) and \(a_0\) take values in the common compact
interval
\[
    [m,M],
    \qquad
    m:=\min\{\gamma,\eps r_-\},
    \quad
    M:=\max\{\Gamma,(1-\eps)M_0+\eps r_+\}.
\]
On this interval the two scalar loss maps are uniformly Lipschitz.
Therefore, the maps
\[
s\mapsto -\log\frac{s}{s+\tau r(y)},
\qquad
s\mapsto -\log\frac{\tau r(y)}{s+\tau r(y)}
\]
are uniformly Lipschitz in $s$ over the admissible range, with a Lipschitz constant
depending only on $(\eps,\tau,r_-,r_+,M_0,\gamma,\Gamma)$.  Hence,
\[
|\ell_a^\pm(X,U)-\ell_{a_0}^\pm(X,U)|
\le
L\,|a(U\mid X)-a_0(U\mid X)| \  \text{a.s.}
\]
Therefore
\begin{align*}
    |g_a(B)-g_{a_0}(B)|&\le
L\left(
(1-\eps)|a(Y\mid X)-a_0(Y\mid X)|
+\eps |a(\widetilde Y\mid X)-a_0(\widetilde Y\mid X)|
\right)\\
&+L\left(
\sum_{j=1}^{\tau}|a(Z_j\mid X)-a_0(Z_j\mid X)|
\right).
\end{align*}
Since all terms are uniformly bounded, this also gives the almost-sure bound
\[
|g_a(B)-g_{a_0}(B)|\le B_h.
\]
For the second moment, use $(u_1+\cdots+u_m)^2\le m\sum_i u_i^2$.  Conditional on
$X=x$, the $Y$-term is integrated against $K_0(x,dy)=k_0(y\mid x)\lambda(dy)$,
and the reference terms are integrated against $\nu(dy)=r(y)\lambda(dy)$.  Since
$k_0\le M_0$ and $r\le r_+$, all these integrals are bounded by a constant times
\[
\int |a(y\mid x)-a_0(y\mid x)|^2\,\lambda(dy).
\]
Integrating over $x\sim\mu$ yields
\[
\mathbb E[(g_a(B)-g_{a_0}(B))^2]
\le
C\|a-a_0\|_{2,\mu\lambda}^2.
\]
By the lower calibration inequality in Theorem~\ref{thm:calibration},
\[
\|a-a_0\|_{2,\mu\lambda}^2\le c_-^{-1}\Excess(a),
\]
and hence
\[
\mathbb E[(g_a(B)-g_{a_0}(B))^2]
\le
V_\ell \Excess(a).
\]

\noindent\textit{\underline{Part (iii).}}
Let $h_a:=g_a-g_{a_0}$.  Then $\E h_a=\Excess(a)$.  By part (ii), every centered variable
$h_a(B)-\E h_a(B)$ is bounded in absolute value by $2B_h$ and has variance at most
$V_{\ell}\Excess(a)$.  The standard bounded Bernstein inequality
\citep{boucheron2013concentration,vershynin2018high} gives,
for fixed $a$ and $t>0$,
\[
 \Prob\left(
   \abs{(\widehat{\mathcal R}_n-\mathcal R)(a)
        -(\widehat{\mathcal R}_n-\mathcal R)(a_0)}
   >
   \sqrt{\frac{2V_{\ell}\Excess(a)t}{n}}
   +\frac{4B_h t}{3n}
 \right)
 \le 2e^{-t}.
\]
The inequality $\sqrt{xy}\le x/8+2y$ for $x,y\ge0$ implies that the deviation on the
left is at most
\[
        \frac14\Excess(a)+C_0\frac{t}{n}
\]
outside the same exceptional event, for a constant $C_0$ depending only on $V_{\ell}$
and $B_h$.

Let $a_1,\ldots,a_N\in\A_n$ be an internal sup-norm
$\delta$-cover of $\A_n$, where
\[
    N=N^{\mathrm{int}}_\infty(\delta,\A_n).
\]
Set
\[
    u:=\log\frac{2N}{\eta}.
\]
Applying the preceding bound with $t=u$ and taking a union bound over the
cover, we obtain an event $\mathcal E$ with $\Prob(\mathcal E)\ge 1-\eta$
such that, on $\mathcal E$, simultaneously for $j=1,\ldots,N$,
\begin{equation}
 \abs{(\widehat{\mathcal R}_n-\mathcal R)(a_j)
        -(\widehat{\mathcal R}_n-\mathcal R)(a_0)}
 \le \frac14\Excess(a_j)+C_0\frac{u}{n}.
 \label{eq:net-fast-event}
\end{equation}
\noindent Work on the event $\mathcal E$.  For any $a\in\A_n$, choose a cover point
$\bar a\in\{a_1,\ldots,a_N\}$ such that
\[
    \norm{a-\bar a}_\infty\le\delta .
\]
By Lemma~\ref{lem:loss-cover},
\begin{equation}
        \abs{\mathcal R(a)-\mathcal R(\bar a)}
        \le L_{\ell}\delta,
        \qquad
        \abs{\widehat{\mathcal R}_n(a)-\widehat{\mathcal R}_n(\bar a)}
        \le L_{\ell}\delta .
\label{eq:cover-lipschitz-risk}
\end{equation}

Let $\bar a_n$ be a cover point of $\widehat a_n$, and fix an arbitrary
comparator $a^\star\in\A_n$ with cover point $\bar a^\star$.  By
\eqref{eq:cover-lipschitz-risk} and the $\delta_{\mathrm{opt}}$-optimality
of $\widehat a_n$,
\begin{align}
 \widehat{\mathcal R}_n(\bar a_n)
 &\le
  \widehat{\mathcal R}_n(\widehat a_n)+L_{\ell}\delta \notag \\
 &\le
 \widehat{\mathcal R}_n(a^\star)+L_{\ell}\delta+\delta_{\mathrm{opt}}
 \notag \\
 &\le
 \widehat{\mathcal R}_n(\bar a^\star)+2L_{\ell}\delta+\delta_{\mathrm{opt}} .
\label{eq:empirical-net-comparison}
\end{align}

We now compare the excess risks on the net.  Applying
\eqref{eq:net-fast-event} to $\bar a_n$ gives
\begin{align}
 \Excess(\bar a_n)
 &=
 \mathcal R(\bar a_n)-\mathcal R(a_0) \notag \\
 &\le
 \widehat{\mathcal R}_n(\bar a_n)-\widehat{\mathcal R}_n(a_0)
 +
 \frac14\Excess(\bar a_n)+C_0\frac{u}{n}.
\label{eq:excess-net-an}
\end{align}
Combining \eqref{eq:excess-net-an} with
\eqref{eq:empirical-net-comparison}, we obtain
\begin{align}
 \Excess(\bar a_n)
 &\le
 \widehat{\mathcal R}_n(\bar a^\star)-\widehat{\mathcal R}_n(a_0)
 +2L_{\ell}\delta+\delta_{\mathrm{opt}}
 +\frac14\Excess(\bar a_n)+C_0\frac{u}{n}.
\label{eq:excess-net-before-comparator}
\end{align}
Applying \eqref{eq:net-fast-event} to $\bar a^\star$ gives
\begin{equation}
 \widehat{\mathcal R}_n(\bar a^\star)-\widehat{\mathcal R}_n(a_0)
 \le
 \Excess(\bar a^\star)
 +
 \frac14\Excess(\bar a^\star)
 +
 C_0\frac{u}{n}.
\label{eq:comparator-net-bound}
\end{equation}
Substituting \eqref{eq:comparator-net-bound} into
\eqref{eq:excess-net-before-comparator} yields
\begin{align}
 \Excess(\bar a_n)
 &\le
 \frac54\Excess(\bar a^\star)
 +2L_{\ell}\delta
 +\delta_{\mathrm{opt}}
 +\frac14\Excess(\bar a_n)
 +2C_0\frac{u}{n}.
\end{align}
Rearranging gives
\begin{equation}
 \Excess(\bar a_n)
 \le
 \frac53\Excess(\bar a^\star)
 +\frac83L_{\ell}\delta
 +\frac43\delta_{\mathrm{opt}}
 +\frac83C_0\frac{u}{n}.
\label{eq:net-oracle-intermediate}
\end{equation}

By \eqref{eq:cover-lipschitz-risk},
\begin{equation}
    \Excess(\widehat a_n)
    \le
    \Excess(\bar a_n)+L_{\ell}\delta,
    \qquad
    \Excess(\bar a^\star)
    \le
    \Excess(a^\star)+L_{\ell}\delta .\label{eq:excess-return-from-net}
\end{equation}
Combining \eqref{eq:net-oracle-intermediate} and
\eqref{eq:excess-return-from-net}, we obtain
\begin{equation}
 \Excess(\widehat a_n)
 \le
 \frac53\Excess(a^\star)
 +\frac{16}{3}L_{\ell}\delta
 +\frac43\delta_{\mathrm{opt}}
 +\frac83C_0\frac{u}{n}. \notag
\end{equation}
Since $a^\star\in\A_n$ was arbitrary, taking the infimum over
$a^\star\in\A_n$ and using $u=\log(2N/\eta)$ gives
\begin{equation}
\Excess(\widehat a_n)
\le
\frac53\inf_{a\in\A_n}\Excess(a)
+
C_{\mathrm{fast}}
\left[
   L_{\ell}\delta
   +\delta_{\mathrm{opt}}
   +\frac{\log\{2N^{\mathrm{int}}_\infty(\delta,\A_n)/\eta\}}{n}
\right],
\notag
\end{equation}
with probability at least $1-\eta$.  Enlarging $C_{\mathrm{fast}}$ if
necessary, this can be written as
\begin{equation}
\Excess(\widehat a_n)
\le
C_{\mathrm{fast}}
\left[
   \inf_{a\in\A_n}\Excess(a)
   +L_{\ell}\delta
   +\delta_{\mathrm{opt}}
   +\frac{\log\{2N^{\mathrm{int}}_\infty(\delta,\A_n)/\eta\}}{n}
\right]. \label{eq:fast-oracle-risk-absorbed}
\end{equation}
By Theorem~\ref{thm:calibration}(ii), on the same event,
\[
   \norm{\widetilde k_{\widehat a_n}-k_0}_{2,\mu\lambda}^2
   \le
   \frac{1}{(1-\eps)^2c_-}\Excess(\widehat a_n).
\]
Combining this inequality with \eqref{eq:fast-oracle-risk-absorbed} yields
\[
   \norm{\widetilde k_{\widehat a_n}-k_0}_{2,\mu\lambda}^2
   \le
   \frac{C_{\mathrm{fast}}}{(1-\eps)^2c_-}
   \left[
      \inf_{a\in\A_n}\Excess(a)
      +L_{\ell}\delta
      +\delta_{\mathrm{opt}}
      +
      \frac{
        \log\{2N^{\mathrm{int}}_\infty(\delta,\A_n)/\eta\}
      }{n}
   \right].
\]
All numerical constants are absorbed into $C_{\mathrm{fast}}$, which depends
only on the boundedness constants and on $\tau$.

\subsection{Proof of Theorem~\ref{thm:markovization}}
\label{app:markov}

Define
\[
    c_g(x):=\int g^+(y\mid x)\,\lambda(dy).
\]
By the assumed \(L^1(\lambda)\)-integrability of \(g(\cdot\mid x)\), we have
\(0\le c_g(x)<\infty\) for every \(x\).  Joint measurability of
\(\mathfrak M g\) follows because \(g^+\) is jointly measurable and
\(x\mapsto c_g(x)\) is measurable; the ratio \(g^+(y\mid x)/c_g(x)\) is
measurable on the set \(\{x:c_g(x)>0\}\), and the fallback \(r(y)\) is jointly
measurable.  In either case, \(\mathfrak M g(\cdot\mid x)\ge0\) and
\[
    \int \mathfrak M g(y\mid x)\,\lambda(dy)=1.
\]
Thus \(\mathfrak M g\) is a transition density.

For the \(L^1\) bound, fix \(x\).  If \(c_g(x)=0\), then \(g^+=0\)
\(\lambda\)-a.e., hence \(g\le0\) \(\lambda\)-a.e.  Since \(k\) is a
transition density,
\[
    \int |g-k|\,d\lambda
    \ge
    \int k\,d\lambda
    =
    1.
\]
On the other hand, \(\mathfrak M g=r\) is a probability density, and therefore
\[
    \int |\mathfrak M g-k|\,d\lambda
    \le
    \int \mathfrak M g\,d\lambda+\int k\,d\lambda
    =
    2
    \le
    2\int |g-k|\,d\lambda.
\]

Now suppose \(c_g(x)>0\).  Write
\[
    \delta:=\int |g-k|\,d\lambda,
    \qquad
    \delta_+:=\int |g^+-k|\,d\lambda.
\]
Since \(k\ge0\), replacing \(g\) by \(g^+\) cannot increase the distance to
\(k\).  Hence
\[
    \delta_+\le\delta.
\]
Then
\[
\begin{aligned}
    \int\left|\frac{g^+}{c_g}-k\right|\,d\lambda
    &\le
    \int\left|\frac{g^+}{c_g}-g^+\right|\,d\lambda
    +
    \int |g^+-k|\,d\lambda \\
    &=
    |1-c_g|+\delta_+.
\end{aligned}
\]
Moreover,
\[
    |1-c_g|
    =
    \left|\int k\,d\lambda-\int g^+\,d\lambda\right|
    \le
    \int |k-g^+|\,d\lambda
    =
    \delta_+.
\]
Therefore
\[
    \int |\mathfrak M g-k|\,d\lambda
    \le
    2\delta_+
    \le
    2\delta.
\]
This proves the pointwise \(L^1\) bound.  Integrating with respect to \(\mu\)
gives
\[
    \|\mathfrak M g-k\|_{1,\mu\lambda}
    \le
    2\|g-k\|_{1,\mu\lambda}.
\]

Taking \(g=\widetilde k_{\widehat a_n}\) and \(k=k_0\), we get
\[
    \|\widehat k_n-k_0\|_{1,\mu\lambda}
    \le
    2\|\widetilde k_{\widehat a_n}-k_0\|_{1,\mu\lambda}.
\]
Since total variation is one half of the \(L^1\) distance between densities,
\[
    d_{\mu,\TV}(\widehat K_n,K_0)
    \le
    \|\widetilde k_{\widehat a_n}-k_0\|_{1,\mu\lambda}.
\]
By Cauchy--Schwarz,
\[
    d_{\mu,\TV}(\widehat K_n,K_0)^2
    \le
    \|\widetilde k_{\widehat a_n}-k_0\|_{2,\mu\lambda}^2.
\]
The bound \eqref{eq:markovized-fast-rate} follows from
Theorem~\ref{thm:oracle}(iii).

\subsection{Proof of Theorem~\ref{thm:holder-minimax}}
\label{app:rate-minimax}
Set $\overline d:=2d$.  We first prove the high-probability upper bound.

Let
\[
        S_n:=\left\lceil n^{\overline d/(2\beta+\overline d)}\right\rceil,
        \qquad
        \A_n:=\A(S_n).
\]
By Assumption~\ref{ass:holder-relu}, the class $\A_n$ consists of jointly
measurable functions taking values in $[\gamma,\Gamma]$, so
Assumption~\ref{ass:candidate-bounds} holds for $\A_n$.  The same assumption
also gives an approximant $a_n^\circ\in\A_n$ satisfying
\[
        \norm{a_n^\circ-a_0}_{\infty}
        \le
        C_{\mathrm{app}}S_n^{-\beta/\overline d}
        (\log S_n)^{\kappa_{\mathrm{app}}}.
\]
Since $\mu\otimes\lambda$ is a probability measure, this implies
\[
        \norm{a_n^\circ-a_0}_{2,\mu\lambda}^2
        \le
        C_{\mathrm{app}}^2
        S_n^{-2\beta/\overline d}
        (\log S_n)^{2\kappa_{\mathrm{app}}}.
\]
By the upper calibration bound in Theorem~\ref{thm:calibration}(ii),
\[
        \inf_{a\in\A_n}\Excess(a)
        \le
        c_+\inf_{a\in\A_n}
        \norm{a-a_0}_{2,\mu\lambda}^2
        \le
        C
        S_n^{-2\beta/\overline d}
        (\log S_n)^{2\kappa_{\mathrm{app}}}.
\]
Here and below $C<\infty$ denotes a generic constant independent of $n$ and
$\eta$.

Apply Theorem~\ref{thm:oracle}(iii) with covering radius $\delta=n^{-1}$.
With probability at least $1-\eta$,
\[
\Excess(\widehat a_n)
\le
C_{\mathrm{fast}}
\left[
   \inf_{a\in\A_n}\Excess(a)
   +L_{\ell}n^{-1}
   +\delta_{\mathrm{opt}}
   +
   \frac{
      \log\{2N^{\mathrm{int}}_\infty(n^{-1},\A_n)/\eta\}
   }{n}
\right].
\]
By Assumption~\ref{ass:holder-relu},
\[
        \log N^{\mathrm{int}}_\infty(n^{-1},\A_n)
        \le
        C_{\mathrm{ent}}S_n(\log n)^{\kappa_{\mathrm{ent}}},
        \qquad n\ge3.
\]
Therefore, on the same event,
\[
\Excess(\widehat a_n)
\le
C
\left[
        S_n^{-2\beta/\overline d}
        (\log S_n)^{2\kappa_{\mathrm{app}}}
        +
        \frac{1}{n}
        +
        \delta_{\mathrm{opt}}
        +
        \frac{S_n(\log n)^{\kappa_{\mathrm{ent}}}}{n}
        +
        \frac{\log(1/\eta)}{n}
\right].
\]
The choice of $S_n$ gives
\[
        S_n^{-2\beta/\overline d}
        \le
        C n^{-2\beta/(2\beta+\overline d)},
        \qquad
        \frac{S_n}{n}
        \le
        C n^{-2\beta/(2\beta+\overline d)}.
\]
Since $2\beta/(2\beta+\overline d)<1$, the term $n^{-1}$ is also bounded by a
constant multiple of $n^{-2\beta/(2\beta+\overline d)}$ for $n\ge3$.  Absorbing
all logarithmic powers into a single exponent $\kappa<\infty$, we obtain
\[
\Excess(\widehat a_n)
\le
C
\left[
        n^{-2\beta/(2\beta+\overline d)}(\log n)^\kappa
        +
        \delta_{\mathrm{opt}}
        +
        \frac{\log(1/\eta)}{n}
\right].
\]
By Theorem~\ref{thm:calibration}(ii),
\[
        \norm{\widetilde k_n-k_0}_{2,\mu\lambda}^2
        \le
        \frac{1}{(1-\eps)^2c_-}
        \Excess(\widehat a_n).
\]
Thus, enlarging $C$ if necessary,
\[
        \norm{\widetilde k_n-k_0}_{2,\mu\lambda}^2
        \le
        C
        \left[
          n^{-2\beta/(2\beta+\overline d)}(\log n)^\kappa
          +
          \delta_{\mathrm{opt}}
          +
          \frac{\log(1/\eta)}{n}
        \right].
\]

For the Markovized estimator, we use Theorem~\ref{thm:markovization} directly.
Applying \eqref{eq:markovized-fast-rate} with $\A_n=\A(S_n)$ and
$\delta=n^{-1}$ gives, on the same event,
\[
 d_{\mu,\TV}(\widehat K_n,K_0)^2
 \le
 \frac{C_{\mathrm{fast}}}{(1-\eps)^2c_-}
        \left[
          \inf_{a\in\A_n}\Excess(a)
          +L_{\ell}n^{-1}
          +\delta_{\mathrm{opt}}
          +\frac{
             \log\{2N^{\mathrm{int}}_\infty(n^{-1},\A_n)\}
             +\log(2/\eta)
           }{n}
        \right].
\]
Using the same approximation and entropy estimates as above, we obtain
\[
        d_{\mu,\TV}(\widehat K_n,K_0)^2
        \le
        C
        \left[
          n^{-2\beta/(2\beta+\overline d)}(\log n)^\kappa
          +
          \delta_{\mathrm{opt}}
          +
          \frac{\log(1/\eta)}{n}
        \right].
\]
Combining the last two displays and recalling that $\overline d=2d$, we obtain
\[
\norm{\widetilde k_n-k_0}_{2,\mu\lambda}^2
\vee
d_{\mu,\TV}(\widehat K_n,K_0)^2
\le
C
\left[
        n^{-2\beta/(2\beta+2d)}(\log n)^\kappa
        +
        \delta_{\mathrm{opt}}
        +
        \frac{\log(1/\eta)}{n}
\right].
\]
This proves the high-probability upper bound.

It remains to prove the minimax lower bound.  It suffices to prove the lower
bound on the submodel where $\mu=\lambda$ is Lebesgue probability measure on
$[0,1]^d$ and $r\equiv1$, since a lower bound over a subclass is also a lower
bound for the full class.

Choose nonzero functions
$\varphi\in C_c^\infty((0,1)^d)$ and
$\psi\in C_c^\infty((0,1)^d)$ such that
\[
        \int_{[0,1]^d}\psi(y)\,dy=0.
\]
Set
\[
        b(x,y):=\varphi(x)\psi(y).
\]
Then $b\in C_c^\infty((0,1)^{2d})$, $b\not\equiv0$, and
\[
        \int_{[0,1]^d} b(x,y)\,dy=0
        \qquad\text{for every }x.
\]
For sufficiently small $h>0$, choose $J_h\ge c h^{-2d}$ disjoint translates of
the support of $b$ inside $[0,1]^{2d}$ and define
\[
        b_{j,h}(z)
        :=
        h^\beta b\left(\frac{z-z_j}{h}\right),
        \qquad
        z=(x,y),
        \qquad
        j=1,\ldots,J_h.
\]
The supports are disjoint, and
\[
        \int b_{j,h}(x,y)\,dy=0
        \qquad\text{for every }x,
\]
while
\[
        \norm{b_{j,h}}_2^2
        =
        h^{2\beta+2d}\norm{b}_2^2.
\]
The standard separated-bump scaling estimates imply that
\[
        \sum_{j=1}^{J_h}\theta_j b_{j,h}
\]
has uniformly bounded $\mathcal H_{2d}^{\beta}$-norm over all
$\theta\in\{0,1\}^{J_h}$.

Choose $\omega>0$ sufficiently small, depending only on
$(d,\beta,R,m,M)$, so that for all sufficiently small $h$,
\[
        k_\theta(x,y)
        :=
        1+\omega\sum_{j=1}^{J_h}\theta_j b_{j,h}(x,y),
        \qquad
        \theta\in\{0,1\}^{J_h},
\]
belongs to $\K_\beta(R,m,M)$.  The zero-integral condition gives
\[
        \int_{[0,1]^d} k_\theta(y\mid x)\,dy=1
        \qquad\text{for every }x,
\]
and the smallness of $\omega$ ensures
\[
        m\le k_\theta(y\mid x)\le M.
\]

Let $P_\theta$ be the one-sample law of $(X,Y)$ under $k_\theta$, and let
$P_\theta^n$ be the law of $n$ independent transition pairs.  If
$\theta^{(j)}$ differs from $\theta$ only in the $j$-th coordinate, then
\[
        k_\theta-k_{\theta^{(j)}}=\pm\omega b_{j,h}.
\]
Since $k_{\theta^{(j)}}\ge m$,
\[
\KL(P_\theta\|P_{\theta^{(j)}})
\le
\int
\frac{(k_\theta-k_{\theta^{(j)}})^2}{k_{\theta^{(j)}}}
\,d\mu\,d\lambda
\le
m^{-1}\omega^2\norm{b_{j,h}}_2^2.
\]
By tensorization,
\[
\KL(P_\theta^n\|P_{\theta^{(j)}}^n)
\le
n m^{-1}\omega^2\norm{b}_2^2 h^{2\beta+2d}.
\]
Choose
\[
        h_n:=c_h n^{-1/(2\beta+2d)}
\]
with $c_h>0$ sufficiently small so that the preceding KL divergence is bounded
by a sufficiently small absolute constant.  By Pinsker's inequality, neighboring
experiments have total variation distance bounded away from one.  Assouad's
lemma then gives
\[
\inf_{\widehat k}
\sup_{\theta\in\{0,1\}^{J_{h_n}}}
\E_\theta
\norm{\widehat k-k_\theta}_{2,\mu\lambda}^2
\ge
c\,\omega^2 J_{h_n} h_n^{2\beta+2d}.
\]
Since $J_{h_n}\ge c h_n^{-2d}$,
\[
        J_{h_n}h_n^{2\beta+2d}
        \ge
        c h_n^{2\beta}.
\]
Therefore
\[
\inf_{\widehat k}
\sup_{\theta\in\{0,1\}^{J_{h_n}}}
\E_\theta
\norm{\widehat k-k_\theta}_{2,\mu\lambda}^2
\ge
c h_n^{2\beta}
=
c n^{-2\beta/(2\beta+2d)}.
\]
Since the finite family $\{k_\theta:\theta\in\{0,1\}^{J_{h_n}}\}$ is contained
in $\K_\beta(R,m,M)$, it follows that
\[
\inf_{\widehat k}
\sup_{K\in\K_\beta(R,m,M)}
\E_K
\norm{\widehat k-k}_{2,\mu\lambda}^2
\ge
c n^{-2\beta/(2\beta+2d)}.
\]
Finally, allowing the estimator to use independently generated reference
variables does not improve the minimax risk: conditioning such an estimator on
the observed transition pairs and applying Jensen's inequality gives an
estimator based only on the transition pairs with no larger risk.  Hence the
same lower bound holds in the augmented experiment.

The upper and lower bounds together show that the squared
$L^2(\mu\otimes\lambda)$ rate is minimax near-optimal up to logarithmic
factors.  The squared TV upper bound follows from the same estimator through
Theorem~\ref{thm:markovization}, and no separate TV minimax lower bound is
claimed.

\subsection{Proof of Theorem~\ref{thm:finite-horizon-transfer-main}}\label{app:dynamics}
This subsection proves the finite-horizon dynamic transfer result used in
Section~\ref{sec:dynamics}.  The proof is a standard occupancy-weighted perturbation
argument, followed by the calibration and Markovization bounds from the main text.

Before presenting the complete proof of Theorem~\ref{thm:finite-horizon-transfer-main}, we first introduce the following lemma.

\begin{lemma}[Occupancy-weighted finite-horizon perturbation]
\label{thm:occupancy}
Let $K$ and $L$ be Markov kernels, let $\xi$ be an initial law, and set
\[
        e_{K,L}(x):=\TV(K(x,\cdot),L(x,\cdot)).
\]
For every integer $m\ge1$,
\[
        \TV(\xi K^m,\xi L^m)
        \le \sum_{s=0}^{m-1}\int e_{K,L}(x)\,(\xi K^s)(\dd x).
\]
The same right-hand side also bounds the total variation distance between the
length-$m$ path laws generated from the common initial law $\xi$ and kernels $K$
and $L$.
\end{lemma}

\begin{proof}
For the marginal law, write
\[
 \TV(\xi K^m,\xi L^m)
 \le
 \TV(\xi K^m,\xi K^{m-1}L)
 +\TV(\xi K^{m-1}L,\xi L^m).
\]
The first term is bounded by
\[
        \int e_{K,L}(x)\,(\xi K^{m-1})(\dd x),
\]
and the second is at most $\TV(\xi K^{m-1},\xi L^{m-1})$ by TV contraction under
Markov kernels.  Iterating this recursion from $m$ down to $1$ gives the displayed
bound.

For path laws, let $P_K^{(m)}$ and $P_L^{(m)}$ denote the laws of
$(X_0,\ldots,X_m)$ generated from $\xi$ under $K$ and $L$.  Let $Q_j$ be the path
law using $K$ for the first $j$ transitions and $L$ for the remaining $m-j$
transitions, so that $Q_0=P_L^{(m)}$ and $Q_m=P_K^{(m)}$.  The neighboring laws
$Q_j$ and $Q_{j+1}$ agree up to time $j$, where the state has law $\xi K^j$.
Replacing the next conditional kernel from $L$ to $K$ changes the joint law by at
most
\[
        \int e_{K,L}(x)\,(\xi K^j)(\dd x),
\]
and appending the same remaining $L$-transitions is a TV contraction.  Summing over
$j=0,\ldots,m-1$ proves the path-law bound.
\end{proof}

\noindent We now present the complete proof of Theorem~\ref{thm:finite-horizon-transfer-main}.

\noindent\textit{\underline{Part (i).}}
Apply Lemma~\ref{thm:occupancy} with $K=K_0$ and $L=\widehat K$.  For
\[
        e(x):=\TV(K_0(x,\cdot),\widehat K(x,\cdot)),
\]
we obtain
\[
  \TV(\xi K_0^m,\xi\widehat K^m)
  \le \sum_{s=0}^{m-1}\int e(x)\,(\xi K_0^s)(\dd x).
\]
If $f_s=\dd(\xi K_0^s)/\dd\mu$ and
$C_s=\esssup_\mu f_s$, then
\[
  \int e(x)\,(\xi K_0^s)(\dd x)
  =
  \int e(x)f_s(x)\,\mu(\dd x)
  \le C_s\int e(x)\,\mu(\dd x).
\]
Since $\int e\,d\mu=d_{\mu,\TV}(K_0,\widehat K)$, summing over
$s=0,\ldots,m-1$ gives the $L^1$-coverage bound
\[
  \TV(\xi K_0^m,\xi\widehat K^m)
  \le
  \paren{\sum_{s=0}^{m-1}C_s}d_{\mu,\TV}(K_0,\widehat K).
\]

\noindent\textit{\underline{Part (ii).}}
For the $L^2$ version, $0\le f_s\le C_s$ and $\int f_s\,d\mu=1$ imply
\[
        \|f_s\|_{L^2(\mu)}^2\le C_s.
\]
Thus Cauchy--Schwarz gives
\[
  \int e(x)\,(\xi K_0^s)(\dd x)
  \le \sqrt{C_s}\,\|e\|_{L^2(\mu)}
  \le \sqrt{C_{\xi,\mu,m}}\,\|e\|_{L^2(\mu)},
\]
where $C_{\xi,\mu,m}:=\max_{0\le s<m}C_s$.  Summing over $s$ yields
\[
  \TV(\xi K_0^m,\xi\widehat K^m)
  \le
  m\sqrt{C_{\xi,\mu,m}}\,\|e\|_{L^2(\mu)}.
\]
The same argument applied to the path-law part of Lemma~\ref{thm:occupancy}
gives the identical bound for length-$m$ path laws.

\noindent\textit{\underline{Part (iii).}}
For normalized occupation measures,
\[
  \bar\Gamma_m^K:=\frac1m\sum_{t=0}^{m-1}\xi K^t,
\]
the triangle inequality and the preceding marginal bound give
\begin{align*}
  \TV(\bar\Gamma_m^{K_0},\bar\Gamma_m^{\widehat K})
  &\le
  \frac1m\sum_{t=0}^{m-1}\TV(\xi K_0^t,\xi\widehat K^t)\\
  &\le
  \frac1m\sum_{t=0}^{m-1}
       t\sqrt{C_{\xi,\mu,m}}\,\|e\|_{L^2(\mu)}\\
  &=
  \frac{m-1}{2}\sqrt{C_{\xi,\mu,m}}\,\|e\|_{L^2(\mu)}.
\end{align*}

It remains to connect $\|e\|_{L^2(\mu)}$ to contrastive excess risk.  Let
\[
        \widetilde k_{\widehat a}
        :=\frac{\widehat a-\eps r}{1-\eps},
        \qquad
        \widehat k:=\mathfrak M\widetilde k_{\widehat a}.
\]
By Theorem~\ref{thm:markovization}, applied pointwise in $x$ with
$g=\widetilde k_{\widehat a}$ and $k=k_0$,
\[
  e(x)
  =
  \frac12\int |\widehat k(y\mid x)-k_0(y\mid x)|\,\lambda(\dd y)
  \le
  \int |\widetilde k_{\widehat a}(y\mid x)-k_0(y\mid x)|\,\lambda(\dd y).
\]
Because $\lambda$ is a probability measure, Cauchy--Schwarz in $y$ gives
\[
        e(x)^2
        \le
        \int |\widetilde k_{\widehat a}(y\mid x)-k_0(y\mid x)|^2\,\lambda(\dd y).
\]
After integrating in $x$,
\[
        \|e\|_{L^2(\mu)}
        \le
        \|\widetilde k_{\widehat a}-k_0\|_{L^2(\mu\otimes\lambda)}.
\]
The calibration theorem gives
\[
        \|\widetilde k_{\widehat a}-k_0\|_{L^2(\mu\otimes\lambda)}^2
        \le
        \frac{1}{(1-\eps)^2c_-}\Excess(\widehat a).
\]
Combining the last two displays with the $L^2$-coverage bounds proves
\[
  \TV(\xi K_0^m,\xi\widehat K^m)
  \le
  \frac{m\sqrt{C_{\xi,\mu,m}}}{(1-\eps)\sqrt{c_-}}\,
  \Excess(\widehat a)^{1/2},
\]
and the corresponding normalized-occupation bound with $m$ replaced by
$(m-1)/2$.

\begin{corollary}[Invariant law perturbation under contraction]
\label{cor:stationary}
Let $K$ and $L$ have invariant laws $\pi_K$ and $\pi_L$.  If
\[
        \alpha(K):=\sup_{x,x'}\TV(K(x,\cdot),K(x',\cdot))<1,
\]
then
\[
        \TV(\pi_K,\pi_L)
        \le
        \frac{d_{\infty,\TV}(K,L)}{1-\alpha(K)},
        \qquad
        d_{\infty,\TV}(K,L):=\sup_x\TV(K(x,\cdot),L(x,\cdot)).
\]
\end{corollary}

\begin{proof}
By invariance of $\pi_K$ and $\pi_L$,
\[
        \pi_K=\pi_KK,
        \qquad
        \pi_L=\pi_LL.
\]
Therefore, by the triangle inequality,
\[
\begin{aligned}
 \TV(\pi_K,\pi_L)
 &= \TV(\pi_KK,\pi_LL)\\
 &\le
 \TV(\pi_KK,\pi_LK)
 +
 \TV(\pi_LK,\pi_LL).
\end{aligned}
\]
The first term is controlled by the Dobrushin contraction coefficient:
\[
        \TV(\pi_KK,\pi_LK)
        \le
        \alpha(K)\TV(\pi_K,\pi_L).
\]
For the second term, use convexity of total variation under mixtures:
\[
\begin{aligned}
 \TV(\pi_LK,\pi_LL)
 &=
 \TV\left(
      \int K(x,\cdot)\,\pi_L(\dd x),
      \int L(x,\cdot)\,\pi_L(\dd x)
    \right)\\
 &\le
 \int \TV(K(x,\cdot),L(x,\cdot))\,\pi_L(\dd x)\\
 &\le
 d_{\infty,\TV}(K,L).
\end{aligned}
\]
Combining the two bounds gives
\[
        \TV(\pi_K,\pi_L)
        \le
        \alpha(K)\TV(\pi_K,\pi_L)
        +
        d_{\infty,\TV}(K,L).
\]
Since $\alpha(K)<1$, moving the first term to the left-hand side yields
\[
        (1-\alpha(K))\TV(\pi_K,\pi_L)
        \le
        d_{\infty,\TV}(K,L).
\]
Dividing by $1-\alpha(K)$ proves the claim.
\end{proof}

\subsection{Proof of Theorem~\ref{thm:trajectory-effective-sample}}
\label{app:trajectory-effective-sample}

For two sigma-fields $\mathcal G$ and $\mathcal H$, define
\[
 \beta(\mathcal G,\mathcal H)
 :=
 \frac12
 \sup
 \sum_{i,j}
 \left|
   \Prob(G_i\cap H_j)-\Prob(G_i)\Prob(H_j)
 \right|,
\]
where the supremum is over all pairs of finite measurable partitions $(G_i)_i$ and
$(H_j)_j$ of the underlying sample space, with $G_i\in\mathcal G$ and
$H_j\in\mathcal H$.  With this normalization, Berbee's coupling lemma gives a
mismatch probability bounded by the corresponding absolute-regularity coefficient.

Write $M=\lfloor N/q\rfloor$ and define the retained augmented blocks
\[
 B_j^{(q)}
 :=
 \bigl(X_{jq-1},X_{jq},\widetilde Y_j,Z_{j1},\ldots,Z_{j\tau}\bigr),
 \qquad j=1,\ldots,M.
\]
The auxiliary variables are independent across retained transitions and independent of
the trajectory.  Appending such variables does not increase the absolute-regularity
coefficient of the retained sequence.  The $j$th retained transition ends at time
$jq$, while the next retained transition begins at time $(j+1)q-1$, so consecutive
retained blocks are separated by $(j+1)q-1-jq=q-1$ time steps.

By Berbee's coupling lemma for absolutely regular sequences
\citep{berbee1979random,doukhan1994mixing}, there exist independent blocks
$B_1^\star,\ldots,B_M^\star$, each $B_j^\star$ having the same marginal law as
$B_j^{(q)}$, such that
\[
 \Prob\left[
   (B_1^{(q)},\ldots,B_M^{(q)})
   \neq
   (B_1^\star,\ldots,B_M^\star)
 \right]
 \le
 (M-1)\beta_X(q-1).
\]
The assumed thinning condition gives
\[
        (M-1)\beta_X(q-1)\le \eta/2.
\]
By stationarity,
\[
        (X_{jq-1},X_{jq})\sim \mu(\dd x)K_0(x,\dd y),
\]
and the auxiliary variables have exactly the same conditional distribution as in the
independent blocks used in Theorem~\ref{thm:oracle}.  Hence
$B_1^\star,\ldots,B_M^\star$ are i.i.d. copies of that block distribution.

On the coupling event, the thinned trajectory empirical risk and the coupled
independent empirical risk coincide pointwise over $\A_M$.  Hence
$\widehat a_{M,q}^{\mathrm{tr}}$, which is a measurable
$\delta_{\mathrm{opt}}$-approximate minimizer of the thinned trajectory risk, is also
a $\delta_{\mathrm{opt}}$-approximate minimizer of the coupled independent empirical
risk on that event.  The proof of Theorem~\ref{thm:oracle}(iii) constructs a
concentration event, depending only on the independent sample blocks, on which the
deterministic oracle argument applies simultaneously to every measurable
$\delta_{\mathrm{opt}}$-approximate empirical-risk minimizer over $\A_M$.  Apply this
uniform concentration event to the coupled independent blocks with sample size $M$,
candidate class $\A_M$, and confidence level $\eta/2$.  A union bound with the
coupling failure event yields, with probability at least $1-\eta$,
\[
 \Excess(\widehat a_{M,q}^{\mathrm{tr}})
 \le
 C_{\mathrm{or}}\,
 \mathfrak R_{M,q}(\eta,\delta),
\]
after absorbing the numerical factor $5/3$ and the constant
$C_{\mathrm{fast}}$ from Theorem~\ref{thm:oracle}(iii) into
$C_{\mathrm{or}}$.

The calibration consequence gives
\[
 \norm{\widetilde k_{M,q}^{\mathrm{tr}}-k_0}_{2,\mu\lambda}^2
 \le
 \frac{1}{(1-\eps)^2c_-}\,
 \Excess(\widehat a_{M,q}^{\mathrm{tr}})
 \le
 C\,\mathfrak R_{M,q}(\eta,\delta).
\]
By the definition of integrated total variation, Theorem~\ref{thm:markovization},
and Cauchy--Schwarz under the probability measure $\mu\otimes\lambda$,
\[
\begin{aligned}
d_{\mu,\TV}(\widehat K_{M,q}^{\mathrm{tr}},K_0)
&=
\frac12
\norm{\mathfrak M\widetilde k_{M,q}^{\mathrm{tr}}-k_0}_{1,\mu\lambda}
\\
&\le
\norm{\widetilde k_{M,q}^{\mathrm{tr}}-k_0}_{1,\mu\lambda}
\\
&\le
\norm{\widetilde k_{M,q}^{\mathrm{tr}}-k_0}_{2,\mu\lambda}.
\end{aligned}
\]
Squaring gives
\[
 d_{\mu,\TV}(\widehat K_{M,q}^{\mathrm{tr}},K_0)^2
 \le
 C_{\mathrm{ker}}\mathfrak R_{M,q}(\eta,\delta).
\]
The marginal, path-law, and normalized occupation-measure bounds then follow from
Theorem~\ref{thm:finite-horizon-transfer-main}, with constants absorbed into
$C_{\mathrm{dyn}}$.

\subsection{Proof of Corollary~\ref{cor:geometric-trajectory}}
\label{app:geometric-trajectory}

By the definition of $q_{N,\eta}$, we have $q_{N,\eta}\ge2$ and
\[
        B\rho^{q_{N,\eta}-1}
        \le
        \frac{\eta}{2N}.
\]
Since $M=N_{\mathrm{eff}}\le N$, it follows that
\[
        (M-1)\beta_X(q_{N,\eta}-1)
        \le
        NB\rho^{q_{N,\eta}-1}
        \le
        \eta/2.
\]
Thus the thinning condition in Theorem~\ref{thm:trajectory-effective-sample} holds
with $q=q_{N,\eta}$ and $M=N_{\mathrm{eff}}$, and all conclusions of that theorem
apply.  The displayed order of $N_{\mathrm{eff}}$ follows directly from
$q_{N,\eta}=O(\log(N/\eta))$.

\subsection{Proof of Proposition~\ref{prop:limitation}}
\label{subsec:lim}

Let
\[
 K_\rho=
 \begin{pmatrix}
   1/2 & 1/2\\
   1/2 & 1/2
 \end{pmatrix},
 \qquad
 L_\rho=
 \begin{pmatrix}
   1/2 & 1/2\\
   1/10 & 9/10
 \end{pmatrix}.
\]
Both kernels are irreducible.  Their first rows agree, while the total variation
distance between their second rows is $2/5$.  Put
\[
        s_\rho:=\min\{(5/2)\rho,1\},
        \qquad
        \mu_\rho=(1-s_\rho)\delta_0+s_\rho\delta_1.
\]
Then
\[
        d_{\mu_\rho,\TV}(K_\rho,L_\rho)
        =\frac25 s_\rho\le\rho.
\]
The invariant law of $K_\rho$ is $(1/2,1/2)$.  Solving
$\pi_{L_\rho}=\pi_{L_\rho}L_\rho$ gives
\[
        \pi_{L_\rho}(1)=\frac{1/2}{1/2+1/10}=\frac56.
\]
Therefore
\[
 \TV(\pi_{K_\rho},\pi_{L_\rho})
   =\abs{\pi_{L_\rho}(1)-1/2}
   =\frac13
   \ge \frac14.
\]


\bibliography{doeblin_anchored_transition_kernels}

@article{li2026dannce,
  author = {Li, Chenghao and Lin, Yuanyuan},
  title = {A Data-Augmented Contrastive Learning Approach to Nonparametric Density Estimation},
  journal = {Journal of Machine Learning Research},
  volume = {27},
  number = {10},
  pages = {1--47},
  year = {2026},
  url = {https://www.jmlr.org/papers/v27/25-0376.html}
}

@article{gutmann2012nce,
  author = {Gutmann, Michael U. and Hyv{\"a}rinen, Aapo},
  title = {Noise-Contrastive Estimation of Unnormalized Statistical Models, with Applications to Natural Image Statistics},
  journal = {Journal of Machine Learning Research},
  volume = {13},
  number = {11},
  pages = {307--361},
  year = {2012},
  url = {https://www.jmlr.org/papers/v13/gutmann12a.html}
}

@inproceedings{ceylan2018conditional,
  author = {Ceylan, Ciwan and Gutmann, Michael U.},
  title = {Conditional Noise-Contrastive Estimation of Unnormalised Models},
  booktitle = {Proceedings of the 35th International Conference on Machine Learning},
  series = {Proceedings of Machine Learning Research},
  volume = {80},
  pages = {726--734},
  year = {2018},
  publisher = {PMLR},
  url = {https://proceedings.mlr.press/v80/ceylan18a.html}
}

@book{anthony1999neural,
  author = {Anthony, Martin and Bartlett, Peter L.},
  title = {Neural Network Learning: Theoretical Foundations},
  publisher = {Cambridge University Press},
  address = {Cambridge},
  year = {1999},
  doi = {10.1017/CBO9780511624216}
}

@article{bartlett2019nearly,
  author = {Bartlett, Peter L. and Harvey, Nick and Liaw, Christopher and Mehrabian, Abbas},
  title = {Nearly-Tight {VC}-Dimension and Pseudodimension Bounds for Piecewise Linear Neural Networks},
  journal = {Journal of Machine Learning Research},
  volume = {20},
  number = {63},
  pages = {1--17},
  year = {2019},
  url = {https://www.jmlr.org/papers/v20/17-612.html}
}

@article{yarotsky2017relu,
  author = {Yarotsky, Dmitry},
  title = {Error Bounds for Approximations with Deep {ReLU} Networks},
  journal = {Neural Networks},
  volume = {94},
  pages = {103--114},
  year = {2017},
  doi = {10.1016/j.neunet.2017.07.002}
}

@article{schmidthieber2020deep,
  author = {Schmidt-Hieber, Johannes},
  title = {Nonparametric Regression Using Deep Neural Networks with {ReLU} Activation Function},
  journal = {The Annals of Statistics},
  volume = {48},
  number = {4},
  pages = {1875--1897},
  year = {2020},
  doi = {10.1214/19-AOS1875}
}

@article{mitrophanov2005sensitivity,
  author = {Mitrophanov, A. Yu.},
  title = {Sensitivity and Convergence of Uniformly Ergodic Markov Chains},
  journal = {Journal of Applied Probability},
  volume = {42},
  number = {4},
  pages = {1003--1014},
  year = {2005},
  doi = {10.1239/jap/1134587812}
}

@book{meyn2009markov,
  author = {Meyn, Sean P. and Tweedie, Richard L.},
  title = {Markov Chains and Stochastic Stability},
  edition = {2},
  publisher = {Cambridge University Press},
  address = {Cambridge},
  year = {2009}
}

@book{vershynin2018high,
  author = {Vershynin, Roman},
  title = {High-Dimensional Probability: An Introduction with Applications in Data Science},
  publisher = {Cambridge University Press},
  address = {Cambridge},
  year = {2018},
  doi = {10.1017/9781108231596}
}

@article{stone1982rates,
  author = {Stone, Charles J.},
  title = {Optimal Global Rates of Convergence for Nonparametric Regression},
  journal = {The Annals of Statistics},
  volume = {10},
  number = {4},
  pages = {1040--1053},
  year = {1982},
  doi = {10.1214/aos/1176345969}
}

@book{boucheron2013concentration,
  author = {Boucheron, St{\'e}phane and Lugosi, G{\'a}bor and Massart, Pascal},
  title = {Concentration Inequalities: A Nonasymptotic Theory of Independence},
  publisher = {Oxford University Press},
  address = {Oxford},
  year = {2013},
  doi = {10.1093/acprof:oso/9780199535255.001.0001}
}

@article{dobrushin1956central,
  title={Central limit theorem for nonstationary Markov chains. I},
  author={Dobrushin, Roland L},
  journal={Theory of Probability \& Its Applications},
  volume={1},
  number={1},
  pages={65--80},
  year={1956},
  publisher={SIAM}
}

@book{sugiyama2012density,
  author = {Sugiyama, Masashi and Suzuki, Taiji and Kanamori, Takafumi},
  title = {Density Ratio Estimation in Machine Learning},
  publisher = {Cambridge University Press},
  address = {Cambridge},
  year = {2012},
  doi = {10.1017/CBO9781139035613}
}

@article{hall1999methods,
  author = {Hall, Peter and Wolff, Rodney C. L. and Yao, Qiwei},
  title = {Methods for Estimating a Conditional Distribution Function},
  journal = {Journal of the American Statistical Association},
  volume = {94},
  number = {445},
  pages = {154--163},
  year = {1999},
  doi = {10.1080/01621459.1999.10473832}
}

@article{hyndman2002nonparametric,
  author = {Hyndman, Rob J. and Yao, Qiwei},
  title = {Nonparametric Estimation and Symmetry Tests for Conditional Density Functions},
  journal = {Journal of Nonparametric Statistics},
  volume = {14},
  number = {3},
  pages = {259--278},
  year = {2002},
  doi = {10.1080/10485250212374}
}

@article{hall2004nonparametric,
  author = {Hall, Peter and Racine, Jeffrey S. and Li, Qi},
  title = {Cross-Validation and the Estimation of Conditional Probability Densities},
  journal = {Journal of the American Statistical Association},
  volume = {99},
  number = {468},
  pages = {1015--1026},
  year = {2004},
  doi = {10.1198/016214504000000548}
}

@article{hall2005conditional,
  author = {Hall, Peter and Yao, Qiwei},
  title = {Approximating Conditional Distribution Functions Using Dimension Reduction},
  journal = {The Annals of Statistics},
  volume = {33},
  number = {3},
  pages = {1404--1421},
  year = {2005},
  doi = {10.1214/009053604000001282}
}

@article{aitsahalia2002maximum,
  author = {A{\"i}t-Sahalia, Yacine},
  title = {Maximum Likelihood Estimation of Discretely Sampled Diffusions: A Closed-Form Approximation Approach},
  journal = {Econometrica},
  volume = {70},
  number = {1},
  pages = {223--262},
  year = {2002},
  doi = {10.1111/1468-0262.00274}
}

@article{hansen1998spectral,
  author = {Hansen, Lars Peter and Scheinkman, Jos{\'e} A. and Touzi, Nizar},
  title = {Spectral Methods for Identifying Scalar Diffusions},
  journal = {Journal of Econometrics},
  volume = {86},
  number = {1},
  pages = {1--32},
  year = {1998},
  doi = {10.1016/S0304-4076(97)00107-3}
}

@article{chen1998sieve,
  author = {Chen, Xiaohong and Shen, Xiaotong},
  title = {Sieve Extremum Estimates for Weakly Dependent Data},
  journal = {Econometrica},
  volume = {66},
  number = {2},
  pages = {289--314},
  year = {1998},
  doi = {10.2307/2998559}
}

@incollection{chen2007large,
  author = {Chen, Xiaohong},
  title = {Large Sample Sieve Estimation of Semi-Nonparametric Models},
  booktitle = {Handbook of Econometrics},
  editor = {Heckman, James J. and Leamer, Edward E.},
  volume = {6B},
  chapter = {76},
  pages = {5549--5632},
  publisher = {Elsevier},
  year = {2007},
  doi = {10.1016/S1573-4412(07)06076-X}
}

@article{li2022minimax,
  title={Minimax optimal conditional density estimation under total variation smoothness},
  author={Li, Michael and Neykov, Matey and Balakrishnan, Sivaraman},
  journal={Electronic Journal of Statistics},
  volume={16},
  number={2},
  pages={3937--3972},
  year={2022},
  publisher={The Institute of Mathematical Statistics and the Bernoulli Society}
}

@techreport{bishop1994mixture,
  author = {Bishop, Christopher M.},
  title = {Mixture Density Networks},
  institution = {Aston University},
  type = {Neural Computing Research Group Report},
  number = {NCRG/94/004},
  year = {1994}
}

@article{trippe2018conditional,
  author = {Trippe, Brian L. and Turner, Richard E.},
  title = {Conditional Density Estimation with Bayesian Normalising Flows},
  journal = {arXiv preprint arXiv:1802.04908},
  year = {2018},
  doi = {10.48550/arXiv.1802.04908}
}

@article{izbicki2017flexcode,
    title={Converting high-dimensional regression to high-dimensional conditional density estimation},
    author={Izbicki, Rafael and B. Lee, Ann},
    journal={Electronic Journal of Statistics},
    volume={11},
    year={2017}
}

@article{gao2022lincde,
  author = {Gao, Zijun and Hastie, Trevor},
  title = {{LinCDE}: Conditional Density Estimation via {Lindsey's} Method},
  journal = {Journal of Machine Learning Research},
  volume = {23},
  number = {52},
  pages = {1--55},
  year = {2022},
  url = {https://www.jmlr.org/papers/v23/21-0840.html}
}

@article{nummelin1978splitting,
  author = {Nummelin, Eero},
  title = {A Splitting Technique for {Harris} Recurrent {Markov} Chains},
  journal = {Zeitschrift f{\"u}r Wahrscheinlichkeitstheorie und Verwandte Gebiete},
  volume = {43},
  number = {4},
  pages = {309--318},
  year = {1978},
  doi = {10.1007/BF00534764}
}

@book{kartashov1996strong,
  author = {Kartashov, Mykola V.},
  title = {Strong Stable {Markov} Chains},
  publisher = {VSP},
  address = {Utrecht},
  year = {1996}
}

@article{ha2018world,
  author = {Ha, David and Schmidhuber, J{\"u}rgen},
  title = {World Models},
  journal = {arXiv preprint arXiv:1803.10122},
  year = {2018},
  doi = {10.48550/arXiv.1803.10122}
}

@inproceedings{hafner2020dream,
  author = {Hafner, Danijar and Lillicrap, Timothy and Ba, Jimmy and Norouzi, Mohammad},
  title = {Dream to Control: Learning Behaviors by Latent Imagination},
  booktitle = {Proceedings of the International Conference on Learning Representations},
  year = {2020},
  url = {https://openreview.net/forum?id=S1lOTC4tDS}
}

@inproceedings{chua2018deep,
  title={Deep reinforcement learning in a handful of trials using probabilistic dynamics models},
  author={Chua, Kurtland and Calandra, Roberto and McAllister, Rowan and Levine, Sergey},
  booktitle={Advances in neural information processing systems},
  volume={31},
  year={2018}
}

@book{wainwright2019high,
  title={High-dimensional statistics: A non-asymptotic viewpoint},
  author={Wainwright, Martin J},
  volume={48},
  year={2019},
  publisher={Cambridge university press}
}

@inproceedings{talvitie2014model,
  title={Model Regularization for Stable Sample Rollouts.},
  author={Talvitie, Erik},
  booktitle={UAI},
  pages={780--789},
  year={2014}
}

@inproceedings{janner2019when,
  author = {Janner, Michael and Fu, Justin and Zhang, Marvin and Levine, Sergey},
  title = {When to Trust Your Model: Model-Based Policy Optimization},
  booktitle = {Advances in Neural Information Processing Systems},
  volume = {32},
  pages = {12498--12509},
  year = {2019},
  url = {https://papers.nips.cc/paper_files/paper/2019/hash/5faf461eff3099671ad63c6f3f094f7f-Abstract.html}
}

@article{bartlett2002rademacher,
  title={Rademacher and gaussian complexities: Risk bounds and structural results},
  author={Bartlett, Peter L and Mendelson, Shahar},
  journal={Journal of machine learning research},
  volume={3},
  number={Nov},
  pages={463--482},
  year={2002}
}

@article{bartlett2005local,
  author = {Bartlett, Peter L. and Bousquet, Olivier and Mendelson, Shahar},
  title = {Local {Rademacher} Complexities},
  journal = {The Annals of Statistics},
  volume = {33},
  number = {4},
  pages = {1497--1537},
  year = {2005},
  doi = {10.1214/009053605000000282}
}

@article{tsybakov2004optimal,
  author = {Tsybakov, Alexandre B.},
  title = {Optimal Aggregation of Classifiers in Statistical Learning},
  journal = {The Annals of Statistics},
  volume = {32},
  number = {1},
  pages = {135--166},
  year = {2004},
  doi = {10.1214/aos/1079120131}
}

@article{koltchinskii2006local,
  author = {Koltchinskii, Vladimir},
  title = {Local {Rademacher} Complexities and Oracle Inequalities in Risk Minimization},
  journal = {The Annals of Statistics},
  volume = {34},
  number = {6},
  pages = {2593--2656},
  year = {2006},
  doi = {10.1214/009053606000001019}
}

@article{lecue2012rates,
  title={Oracle inequalities for cross-validation type procedures},
  author={Lecu{\'e}, Guillaume and Mitchell, Charles},
  journal={Electronic Journal of Statistics},
  volume={6},
  pages={1803--1837},
 year={2012}
}

@book{doukhan1994mixing,
  author    = {Doukhan, Paul},
  title = {Mixing: Properties and Examples},
  series    = {Lecture Notes in Statistics},
  volume    = {85},
  publisher = {Springer},
  address   = {New York},
  year      = {1994},
  doi       = {10.1007/978-1-4612-2642-0}
}

@book{berbee1979random,
  author    = {Berbee, H. C. P.},
  title = {Random Walks with Stationary Increments and Renewal Theory},
  series    = {Mathematical Centre Tracts},
  volume    = {112},
  publisher = {Mathematisch Centrum},
  address   = {Amsterdam},
  year      = {1979}
}

\end{document}